\documentclass[nohyperref]{article}
\usepackage{arxiv}
\usepackage{natbib}
\usepackage[colorlinks,citecolor=blue,urlcolor=blue,bookmarks=false,hypertexnames=true,linkcolor=blue]{hyperref}
\usepackage[english]{babel}
\usepackage{nicefrac}       % compact symbols for 1/2, etc.

\usepackage{microtype}
\usepackage{graphicx}
\usepackage{subfigure}
\usepackage{booktabs} % for professional tables

\usepackage{alias}
\usepackage{amsmath,amsfonts,amssymb}
\usepackage{amsthm}
\usepackage{xcolor}
\usepackage{hyperref}

\usepackage{thm-restate}
\usepackage{graphicx}
\usepackage{hyperref,url}
\usepackage{booktabs}
\usepackage{tabu}

\usepackage{amsmath}
\usepackage{amssymb}
\usepackage{mathtools}
\usepackage{amsthm}

\usepackage[capitalize,noabbrev]{cleveref}

\theoremstyle{plain}

\theoremstyle{definition}

\theoremstyle{remark}

\usepackage[textsize=tiny]{todonotes}

\usepackage{color-edits}
\addauthor{ab}{blue}

\title{On the (Non-)Robustness of Two-Layer Neural Networks in Different Learning Regimes}

\usepackage{afterpage}

\author{%
  \name Elvis Dohmatob$^1$ \email{dohmatob@fb.com}\\
  \name Alberto Bietti$^2$
  \email{alberto.bietti@nyu.edu}\\
    \addr{$1$ Facebook AI Research}\\
  \addr{$2$ New York University}
}
\begin{document}

\maketitle
% \addtocontents{toc}{\protect\setcounter{tocdepth}{0}}






\begin{abstract}
% We consider the problem of learning a quadratic form using a two-layer neural network and establish gaps in the robustness (to perturbations in the input at test time) in the following regimes: (i) at initialization; (i) random features regime where only output layer is trained; (iii) neural tangent regime; and (iv) the SGD regime where all the layers are trained with any linear approximations.
% % Our analysis shows that the linearized regimes (RF and NT) tend to be linear (and therefore more robust than the ground truth) in the under-parametrized regime, and become as robust as the ground-truth degree of over-parametrization is increased.
% For this problem, our analysis allows us to identify No Free Lunch phenomena, whereby test error and robustness are tied together: the one can only decrease at the expense of the other increasing, and vise versa. 
%     % networks trained with SGD with infinite training data...

% \textbf{alternative:}
Neural networks are known to be highly sensitive to adversarial examples. These may arise due to different factors, such as random initialization, or spurious correlations in the learning problem.
To better understand these factors, we provide a precise study of the adversarial robustness in different scenarios, from initialization to the end of training in different regimes, as well as intermediate scenarios, where initialization still plays a role due to “lazy” training. We consider over-parameterized networks in high dimensions with quadratic targets and infinite samples. Our analysis allows us to identify new tradeoffs between approximation (as measured via test error) and robustness, whereby robustness can only get worse when test error improves, and vice versa. We also show how linearized lazy training regimes can worsen robustness, due to improperly scaled random initialization. Our theoretical results are illustrated with numerical experiments.
\end{abstract}

\setcounter{tocdepth}{2}

\tableofcontents

% \mytodo{Missing NeurIPS checklist at end of manuscript (i.e just after "Conclusion" section) !}

% \tableofcontents

% \textbf{Remaining tasks (in decreasing order of importance):}
% \begin{itemize}
  % \item[(0)] \textcolor{red}{\textbf{\textit{XXX TODO: Make space (summarize some stuff, move overly-tech stuff to appendix)!}}}

    %\item[(1)] \textcolor{red}{\textbf{\textit{XXX Include better figures + better captions.}}}
    
    % \item[(2)] \textcolor{red}{\textbf{\textit{XXX TODO: Write soft / high-level introduction (main story, contributions, etc.). Improve on the abstract.}}}

    % \item[(3)] \textcolor{red}{\textbf{\textit{XXX TODO: Write related works section.}}}

%    \item[(4)] \textcolor{red}{\textbf{\textit{XXX TODO: Reorganize manuscript. Move technical details to appendix (especially much of section on lazy RF)}}}
    
   % \item[(5)] \textcolor{red}{\textbf{\textit{XXX TODO: Include illustrative figures / diagrams.}}}
    
%   \textcolor{red}{\textbf{\textit{XXX TODO: (Crucial) I don't know if the terms I'm using "RFL (lazy RF) and "NTL (lazy NT) really make sense or are conceptually nonsensical.}}}

% \end{itemize}

\section{Introduction}

Deep neural networks have enjoyed tremendous practical success in many applications involving high-dimensional data, such as images.
Yet, such models are highly sensitive to small perturbations known as adversarial examples~\cite{szegedy2013intriguing}, which are often imperceptible by humans.
While various strategies such as adversarial training~\cite{madry2017towards} can mitigate this vulnerability empirically, lack of robustness remains highly problematic for many safety-critical applications like autonomous vehicles and health, and motivates a better understanding of the phenomena at play.

Various factors are known to contribute to adversarial examples. In linear models, features that are only weakly correlated with the label, possibly in a spurious manner, may improve prediction accuracy but induce large sensitivity to adversarial perturbations~\cite{tsipras18,sanyal2021benign}.
On the other hand, common neural networks may exhibit high sensitivity to adversarial perturbations at random initialization~\cite{yann2019firstorder,mostrelu2020,bubecksinglestep2021,bartlett2021}.
Trained networks may thus involve multiple sources of vulnerability arising from initialization, training algorithms, as well as the data distribution at hand.

% \subsection{Contributions}
In this paper, we study the interplay between these different factors by analyzing approximation and robustness properties (i.e., stability of predictions, w.r.t perturbations in test data) of two-layer networks in different learning regimes.
We consider two-layer finite-width networks in high dimensions with infinite data, in asymptotic regimes inspired by~\cite{Ghorbani19}.
This allows us to focus on the effects inherent to the data distribution and the inductive bias of architecture and training algorithms, while side-stepping issues due to finite samples.
Following~\cite{Ghorbani19}, we focus on regression settings with structured quadratic target functions, and consider commonly studied training regimes for two-layer networks, namely (i) neural networks with quadratic activations trained with stochastic gradient descent on the population risk, which exhibit a form of feature learning by finding the global optimum; (ii) random features~\cite[RF,][]{rf}, (iii) neural tangent kernel~\cite[NT,][]{jacot18}, as well as (iv) ``lazy'' training~\cite{chizat2018lazy} regimes for RF and NT, where we consider a first-order Taylor expansion of the network around initialization, including the initialization term itself (in contrast to the RF and NT regimes which focus on the first-order term).
Note that, though the theoretical setting is inspired by \cite{Ghorbani19}, our work differs from~\cite{Ghorbani19} in its focus and scope. Indeed, we are concerned with robustness and its interplay with approximation, in different learning regimes, while \cite{Ghorbani19} was only concerned with approximation.
We also note that the lazy/linearized regimes we study as part of this work were not considered by~\cite{Ghorbani19}, and help us highlight the impact of initialization on robustness.
%\AB{je pense que lazy est un truc à part, en tout cas distinct de Ghorbani. Aussi}

% \mytodo{Check citations for RF and NT are complete / accurate.}

% \mytodo{Missing reference for feature learning / lazy NT regime.}\abcomment{feature learning c'est SGD dans notre cas, quand tu apprends des poids adaptés aux données. lazy NT c'est un kernel regime aussi}

% \abdelete{In order to better understand the role of initialization, we also consider “lazy training” regimes, consisting of linearized models around initialization~\cite{chizat2018lazy}.}

\paragraph{Main contributions.}
% We characterize robustness using a Sobolev norm quantity, which is tightly linked to adversarial risk.
Our work establishes theoretical results which uncover novel tradeoffs between approximation (as measured via test error) and robustness that are inherent to all the regimes considered. These tradeoffs appear to be due to misalignment between the target function and the input distribution (or weight distribution) for random features (Section \ref{sec:rf}), or to the inductive bias of fully-trained networks (Sections \ref{sec:sgd} and \ref{sec:init}).
We also show that improperly scaled random initialization can further degrade robustness in lazy/linearized models (Section \ref{sec:nt}), since they might inherit the nonrobustness inherent to random initialization. This raises the question of how large should the initialization be to in order to enhance the robustness of the trained model.
Our theoretical results are empirically verified with extensive numerical experiments on simulated data.

The setting of our work is regression with two-layer neural networks, where we assume access to infinite training data. Thus, the only complexity parameters are the structure of the ground-truth function, the input dimension $d$ and the width of the neural network $m$, assumed to both "large" but proportional to one another. Refer to Section \ref{sec:prelim} for details.
% \begin{restatable}{rmk}{}{}
The infinite-sample setting allows us to focus on the effects inherent to the data distribution and the inductive bias of architecture (choice of activation function) and different learning regimes, while side-stepping issues due to finite samples and label noise. Also note that in this infinite-data setting, label noise provably has no influence on the learned model, in all the learning regimes considered. The observation that there is a tradeoff between robustness and approximation, even in this infinite-sample setting, is one of the surprising findings of our work. This complements related works such as~\cite{lor,bubeck2021universal}, which show that finite training samples with label noise is a possible source of nonrobustness in neural networks.
% \end{restatable}

% \paragraph{Organization of the paper.}

% \mytodo{move what's below to section 2 "Preliminaries"}

% \mytodo{Reorganize / improve intro}

% We aim in this paper to expose gaps between the robustness of RF, NTK, and SGD regimes of neural nets. % Hopefully SGD is the most robust of them all (i.e smallest $\mathfrak{S}(f)^2$), followed by NT.
% \subsection{Problem setup}

\section{Related work}

Various works have theoretically studied adversarial examples and robustness in supervised learning, and the relationship prediction performance.

\cite{tsipras18} considers a specific data distribution where good accuracy implies poor robustness. For example, \cite{goldstein,saeed2018, gilmerspheres18,dohmatob19} show that for high-dimensional data distributions which have concentration property (e.g., multivariate Gaussians, distributions satisfying log-Sobolev inequalities, etc.), an imperfect classifier will admit adversarial examples. \cite{dobriban2020provable} studies trade-offs in Gaussian mixture classification problems, highlighting the impact of class imbalance. On the other hand, \cite{closerlook2020} observed empirically that natural images are well-separated, and so locally-lipschitz classifies shouldn't suffer any kind of test error vs robustness tradeoff. However, gradient-descent is not likely to find such models.
Our work studies regression problems with quadratic targets, and shows that there are indeed trade-offs between test error and robustness which are controlled by the learning algorithm / regime and model.

\cite{yann2019firstorder,mostrelu2020,bubecksinglestep2021,bartlett2021} study adversarial vulnerability of neural networks at initialization, but do not consider the effects of training the model, in contrast to our work.

\cite{schmidt2018,khim2018adversarial,yin2019rademacher,bhattacharjee2020sample,min2021curious,Karbasi2021} study the sample complexity of robust learning. In contrasts, our work focuses on the case of infinite data, so that the only complexity parameters are the input dimension $d$ and the network with $m$. \cite{bhattacharjee2020sample} studies robustness vs accuracy for data distributions which are well-separated (e.g., say the two classes are supported on disjoint balls). The main finding in that paper is that (i) the robustness vs accuracy tradeoff doesn't exist for well-separated datasets. The work also posits that (ii) real-world datasets are well-separated. We think (i) is only an artifact of the well-separatedness assumption (an assumption which fails for Gaussians (as noted in the paper), say, due to infinite support). Also, (ii) is likely due to the fact that most real datasets are limited in sample size, and so, deceptively appear to be well-separated. Indeed, in the real world, there are cats which look like dogs (e.g, Siamese cats), even though such data might be under-represented in ML datasets.

% \mytodo{add~\citep{sanyal2021benign} somewhere}

\cite{RuiqiGao2019,lor,bubeck2021universal} show that over-parameterization may be necessary for robust interpolation in the presence of noise. In contrast, our paper considers a structured problem with noiseless signal and infinite-data $n=\infty$, where the network width $m$ and the input dimension $d$ tend to infinity proportionately. In this under-complete asymptotic setting, our results show a precise picture of the trade-offs between approximation (test error) and robustness in different learning regimes.  Our work nuances this picture by exhibiting a nontrivial interplay between robustness and test error which persists even in the case of infinite samples, where the model isn't affected by label noise.
% \abcomment{Move the last two sentences to related work? Or in the general intro?}

\cite{dohmatob2021fundamental,hassani2022curse} study the tradeoffs between interpolation, predictive performance (test error), and robustness for finite-width over-parameterized networks in kernel regimes with noisy linear target functions. In contrast, we consider structured quadratic target functions and compare different learning settings, including SGD optimization in a non-kernel regime, as well as lazy/linearized models.

\section{Preliminaries}
\label{sec:prelim}

\subsection{Notations}
% \abcomment{je mettrais la subsection notation soit en premier soit en dernier dans la section 2?}
% Before moving forth, let us introduce some technical notation that will be used in this paper.
% \paragraph{Linear algebra.}
The set of integers from $1$ through $d$ will be denoted $[d]$. We will denote the identity matrix of size $d$ by $I_d$. The euclidean norm of a vector $x \in \mathbb R^d$ will be denoted $\|x\|$. The $k$th largest singular-value of a matrix $A$ will be denoted $\sigma_k(A)$, and equals the positive square-root of the $k$th largest eigenvalue $\lambda_k(AA^\top)$ of the psd matrix $AA^\top$. If $A$ itself is psd then $\sigma_k(A) = \lambda_k(A)$ for any $k \in [d]$.
% $\lambda_j(B)$ will denote the $j$th largest eigenvalue of the psd matrix $B$.
The Frobenius norm of $A$ denoted $\|A\|_F$ is defined by $\|A\|_F := \sqrt{\sum_{k=1}^d \sigma_k(A)^2}$. More generally, we define $\|A\|_{F,m} := \sqrt{\sum_{k=1}^{m \land d}\sigma_k(A)^2}$, so that $\|A\|_{F,d} = \|A\|_F$ in particular. Note that $m \mapsto \|A\|_{F,m}$ is a nondecreasing function which is upper-bounded by $\|A\|_F$. The Hadamard / element-wise product of two compatible matrices $A_1$ and $A_2$ will be denoted $A_1 \circ A_2$. 
The squared $L_2$-norm of a function $f:\mathbb R^d \to \mathbb R$ w.r.t the standard $d$-dimensional Gaussian distribution on $N(0,I_d)$ will be denoted $\|f\|_{L^2(N(0,I_d))}^2$, and defined by
% \begin{eqnarray}
$\|f\|_{L^2(N(0,I_d))} := \mathbb E_{x \sim N(0,I_d)}[f(x)^2],$
% \end{eqnarray}
when this integral exists.

\paragraph{Hermite coefficients.}
For any nonnegative integer $k$, let $\mathrm{He}_k:\mathbb R \to \mathbb R$ be the (probabilist's) $k$th Hermite polynomial. For example, note that $\mathrm{He}_0(t) := 0$, $\mathrm{He}_1(t) := t$, $\mathrm{He}_2(t) \equiv t^2-1$, $\mathrm{He}_3(t) := t^3 - 3t$, etc. The sequence $(\mathrm{He}_k)_k$ forms an orthonormal basis for the Hilbert space $L^2 = L^2(N(0,1))$ for functions $\mathbb R \to \mathbb R$ which are square-integrable w.r.t the standard normal distribution $N(0,1)$. Under suitable integrability conditions (refer to Section \ref{subsec:keyassume}), the coefficients of the activation function $\sigma$ in this basis are called its \emph{Hermite coefficients}, denoted $\lambda_k$, and are given by
\begin{eqnarray}
\lambda_k=\lambda_k(\sigma):= \mathbb E_{G \sim N(0,1)}[\sigma(G)\mathrm{He}_k(G)].
\end{eqnarray}

Finally, 
% \begin{eqnarray}
$\|\sigma\|_{L^2(N(0,1))}^2 = \mathbb E_{G \sim N(0,1)}[\sigma(G)^2]$
% \end{eqnarray}
defines the squared $L_2$-norm of $\sigma$ w.r.t the %one-dimensional
standard Gaussian distribution $N(0,1)$. Note that by construction, one has $\|\sigma\|_{L^2(N(0,1))}^2 = \sum_{k=0}^\infty \lambda_k^2(\sigma)$.
% \abcomment{This paragraph is a bit too math heavy, maybe we can move it to the notations section?}

\paragraph{Asymptotics.} The usual notation $\mathcal O_d(1)$ (resp. $\mathcal O_{d,\mathbb P}(1)$) is used to denote a quantity which remains bounded (resp. bounded in probability) in the limit $d \to \infty$. Likewise $o_{d}(1)$ (resp. $o_{d,\mathbb P}(1)$) denotes a quantity which goes to zero (resp. which goes to zero in probability) in the limit $d \to \infty$. As usual, the acronym "a.s" stands for \emph{almost-surely}, while "w.p $p$" stands for \emph{with probability at least $p$}.
% \section{Preliminaries}
% \subsection{Dirichlet energy as a measure of nonrobustness}

\subsection{Ground truth / teacher model.}
We consider the following regression setup proposed in Ghorbani et al ~\cite{Ghorbani19}.
Let~$B$ be a fixed $d \times d$ psd matrix and let $b_0 \in \mathbb R$ be a fixed unknown scalar. Consider the following quadratic target model
\begin{eqnarray}
\label{eq:truemodel}
f_\star(x) := x^\top B x + b_0,\;\forall x \in \mathbb R^d.
\end{eqnarray}
% This corresponds to a \emph{teacher} network with quadratic activation function and hidden neuron with parameters $W^\star=(w^\star_1,\ldots,w^\star_M) \in \mathbb R^{M \times d}$ chosen to be a rank-$M$ Cholesky factor of $B$. Let $\lambda_1(B) \ge \lambda_2(B)\ge \ldots \ge \lambda_r(B)$ be the nonzero eigenvalues of $B$, and let $r \le d$ be its rank.
% \subsection{Two-layer neural network}

We assume the input data is distributed according to $N(0,I_d)$, the standard Gaussian distribution in $d$ dimensions. Thus,
the structure of the problem of learning the ground-truth function $f_\star$ in \eqref{eq:truemodel} is completely determined by the unknown $d \times d$ matrix $B$. We assume an idealized scenario where the learner has access to an infinite number of iid samples of the form $(x,f_\star(x))$ with $x \sim N(0,I_d)$.
% To simplify the picture, we assume that the sample size is $n = \infty$, i.e, the learner has oracle access to iid samples from this distribution.
For simplicity of analysis, we will further assume as in \cite{Ghorbani19} that the ground-truth function $f_\star$ defined in \eqref{eq:truemodel} is centered, i.e $\mathbb E_{x \sim N(0,I_d)}[f_\star(x)] = 0$. This forces the offset $b_0$ to be equal to $-\trace(B)$.

% Even under these simplifications, we shall put to light a number of interesting phenomena which reveal a fundamental tradeoff between test (i.e., approximation) error and robustness to (possibly adversarial) input perturbations.

\subsection{Finite-width two-layer neural network}
% \mytodo{Include illustration of a neural network}
Consider a two-layer neural network 
\begin{eqnarray}
f_{W,z,s}(x) &:= \sum_{j=1}^m z_j \sigma (x^\top w_j) - s,
\label{eq:nn}
\end{eqnarray}
where $m$ is the network width, i.e., the number of hidden neurons each with parameter vector $w_j \in \mathbb R^d$, output weights $z=(z_1,\ldots,z_m) \in \mathbb R^m$, and activation function $\sigma:\mathbb R \to \mathbb R$. We define $W$ as the $m \times d$ matrix with $j$th row $w_j$. The scalar $s$ is an offset which we will sometimes set to zero, in which case we will simply write $f_{W,z,s} := f_{W,z,0}$.
% This neural net template was proposed in Ghorbani et al. as a theoretical testbest for exploring the gap between  linearized regimes (random features, neural tangents, etc.) and models learned via SGD.

\paragraph{Asymptotic regime with infinite data.}
Motivated by~\cite{Ghorbani19}, we will consider the following regime:

\emph{-- Infinite data.} The sample size $n$ is \textit{equal to} $\infty$, i.e., the learner has access to the entire data distribution, allowing us to step-aside issues linked with finite samples.
    % \abcomment{ change to: The learner has access to the entire data distribution, leading us to focus on approximation error?}
    
\emph{-- Proportionate scaling of dimensionality and width.} The input-dimension $d$ and the network width $m$ are \textit{finite}, and large of the same order, i.e.,
\begin{eqnarray}
\label{eq:proportionate}
m,d \to \infty,\; m/d \to \rho \in (0,\infty).
\end{eqnarray}
% \end{restatable}
The parameter $\rho$ (called the \emph{aspect ratio}), which will play a crucial role in the sequel, corresponds to the parametrization rate: $\rho < 1$ corresponds to \emph{under-parametrization}, while
$\rho > 1$ corresponds to \emph{over-parametrization} (more hidden neurons than input dimension). We will also consider the extreme over-parametrization regime corresponding to the double limit: $m,d\to \infty$, $m/d \to \rho$, $\rho\to \infty$.

% \subsection{Objectives and main contributions of the paper}
% Our aim is to analyze the robustness of the model \eqref{eq:nn} in 3 different regimes
% \begin{itemize}
%     \item[(1)] \textbf{SGD regime}, where the network is trained via stochastic gradient descent (SGD).
%     \item[(2)] \textbf{Random features (RF) regime} wherein only the output weight $a$ is optimized, while the hidden weights $w_j$ are frozen at their random values at initialization.
%     \item[(3)] \textbf{Neural Tangent (NT) regime}, wherein only the hidden weights $w_j$ are trained.
%     \item[(4)] \textbf{Init regime}, wherein the neural network is not trained at all, but frozen at the random values of its output and hidden weights at initialization.
% \end{itemize}

% Just as with test error as shown in \cite{Ghorbani19},
The parameter $\rho$ will play a crucial role in our analysis of robustness, similar to~\cite{Ghorbani19} for approximation.
% We wish to understand the interplay between robustness and generalization, and uncover the impact of problem specific parameters like the rate of over-parametrization $\rho$ and the shape of the coefficient matrix $B$ on the picture.

% % \paragraph{Complete picture of robustness of neural networks in different regimes}
% \paragraph{Analytic formula for nonrobustness.}
% \lorem

% \paragraph{Tradeoffs between generalization and robustness.}
% \lorem 

% \paragraph{Lazy regimes and the influence of random inialization.}
% \lorem

\subsection{Metrics for test error and robustness}
\label{subsec:metrics}
\paragraph{Test error.}
The test / approximation error of a function $f:\mathbb R^d \to \mathbb R$ (e.g a neural network) is
\begin{eqnarray}
\label{eq:generr}
\varepsilon_{{\rm test}}(f)
% := \|f-f_\star\|_{L^2(N(0,I_d))}^2
:= \|f-f_\star\|_{L^2(N(0,I_d))}^2 = \mathbb E_{x \sim N(0,I_d)}|f(x)-f_\star(x)|^2,
\end{eqnarray}
and measures how well $f$ approximates the ground-truth function $f^\star$ w.r.t to the data distribution $N(0,I_d)$.
% where $\|g\|_{L^2(N(0,I_d))}^2 := \mathbb E_{x \sim N(0,I_d)}|g(x)|^2$.
It will be instructive to compare the test error of a model with that of the \emph{null predictor}, which outputs $0$ on every input $x$.
% , i.e., $\|f_\star\|_{L^2(N(0,I_d))}^2$.
To this end, we will usually consider the \emph{normalized} test error,
\begin{eqnarray}
\egen(f) := \varepsilon_{{\rm test}}(f)/\|f_\star\|_{L^2(N(0,I_d))}^2.
% \,\text{ where }\|f_\star\|_{L^2(N(0,I_d))}^2 := \mathbb E_{x \sim N(0,I_d)}|f(x)|^2
\end{eqnarray}
% \paragraph{Parametrization rate.}
This quantity was studied in \cite{Ghorbani19} where explicit analytic formulae were obtained in various regimes of interest for two-layer networks: networks fully trained by stochastic gradient-descent (SGD) on the population risk, random features (RF), and neural tangent (NT).
We shall consider these same regimes and establish tradeoffs between test error and robustness of the corresponding models.
% \abdelete{We shall consider these same regimes and study the robustness of the resulting models.}

% \input{table.tex}

\paragraph{Measure of robustness / stability.}
% \abcomment{for consistency, this should also be a paragraph? otherwise, change title of previous section to "Measure of generalization"}
We will measure the robustness of a (weakly) differentiable
% \footnote{We can replace continous differentiability on $\mathbb R^d$ with much weaker notions of differentiability, but such general considerations would get us a bit astray. For example, it is sufficient for to require $f \in \mathcal C^1(\Omega)$, for some measurable $\Omega \subseteq \mathbb R^d$ such that $\mu(\Omega)=1$. For example, a neural net with ReLU activation function, though not studied in this paper.}
function $f:\mathbb R^d \to \mathbb R$ (e.g., the two-layer neural net \eqref{eq:nn}) by the square-root of its Dirichlet energy (aka squared Sobolev-seminorm) $\mathfrak{S}(f)^2$ w.r.t. to a random test point $x \sim N(0,I_d)$, defined by setting
\begin{eqnarray}
\label{eq:sob2}
\mathfrak{S}(f)^2 := \|\nabla_x f\|_{L^2(N(0,I_d))}^2 = \mathbb E_{x \sim N(0,I_d)}\|\nabla_x f(x)\|^2.
\end{eqnarray}
% The factor $1/4$ will help simplify subsequent formula, as we will see later.
The smaller the value of $\mathfrak{S}(f)$, the more robust / stable $f$ is to changes in a typical input data point, namely $x \sim N(0,I_d)$ here.
% This can be seen as the regression analog of the \emph{perturbation stability} metric considered in \cite{xyz} in the classification setting.
We justify the choice of this quantity as a measure of robustness in Appendix~\ref{sec:justice}, where we show that it may be viewed as a first-order approximation of the adversarial risk~\citep{madry2017towards}.
Note that, stability as a measure of robustness has been considered in other works like ~\cite{lor,bubeck2021universal} (for regression settings) and \cite{widerwu2021} (for classification settings).
% We argue that $\mathfrak{S}(f)$ is a more reasonable measure of robustness than say, Lipschitz constant (considered in \cite{lor,bubeck2021universal}), because the former takes into account the distribution of the data, while the latter is a worst-case measure. Finally, note note that $\|f\|_{\Lip} \ge \mathfrak{S}(f)$.
% We also note that this metric has been used empirically as a regularizer to impose robustness~\cite{jacobianreg}.\abcomment{Remove this last sentence?}

It will be convenient to compare the robustness of a model $f$  to that of the baseline quadratic ground-truth function $f_\star$ defined in \eqref{eq:truemodel}.
% \begin{restatable}[Normalized measure of nonrobustness]{df}{}
To this end, the \emph{normalized} measure of robustness of $f$ is % denoted $\erob(f)$ is defined by
\begin{eqnarray}
\label{eq:erob}
\erob(f) := \mathfrak{S}(f)^2/\mathfrak{S}(f_\star)^2.
\end{eqnarray}
% \end{restatable}
The objective of our paper is to study the quantity $\erob(f)$ for neural networks \eqref{eq:nn} in various regimes % (SGD, random features (RF), neural tangent (NT), etc.)
in the limit \eqref{eq:proportionate}, and put to light interesting phenomena. It paints a picture complementary to \cite{Ghorbani19}.
% \paragraph{An analytic formula.}

Let us start by deriving an analytic formula for the robustness measure for the neural network general model \eqref{eq:nn}. This result will be exploited in the sequel in the analysis of the different learning regimes we will consider.
% the study of random features regime and models trained via SGD.
\begin{restatable}{lm}{formula}
For the neural net $f_{W,z,s}$ defined in \eqref{eq:nn}, we have the analytic formula
% \begin{eqnarray}
$\mathfrak{S}(f_{W,z,s})^2 = z^\top C z$, % \sum_{j,k=1}^m a_j a_k (w_j^\top w_k)(w_j^\top \Lambda w_k),
% \label{eq:formula}
% \end{eqnarray}
where $C=C(W)$ is the $m \times m$ psd matrix with entries given by (with $x \sim N(0,I_d)$)  
\begin{eqnarray}
c_{j,k}:=(w_j^\top w_k)\mathbb E_x[\sigma'(x^\top w_j)\sigma'(x^\top w_k)]\,\forall j,k \in [m].
\label{eq:randomCbis}
\end{eqnarray}
In particular, for a quadratic activation $\sigma(t) \equiv t^2 + s$, we have
% \begin{eqnarray}
$c_{j,k} = 4(w_j^\top w_k)^2 \,\,\forall j,k \in [m].$
% \end{eqnarray}
\label{lm:formula}
\end{restatable}
% \begin{restatable}{rmk}{}
% Formula \eqref{eq:formula} is instructive: it says that the (non)robustness of $f_{W,z}$ (as measured by its Dirichlet energy $\mathfrak{S}(f_{W,z})^2$) only depends on the alignment between the (covariance matrix of) the data distribution $\mu$ \abcomment{where is~$\mu$ defined?} and the parameters $w_1,\ldots,w_m$ of the hidden neurons.
% % \end{restatable}

% \input{justification.tex}
% \vspace{-.75cm}
\section{Results for fully-trained networks}
\label{sec:sgd}
Consider a neural network model $f_\sgd:\mathbb R^d \to \mathbb R$ given by
\begin{eqnarray}
f_\sgd(x) := \sum_{j=1}^m (x^\top w_j)^2 + s.
\label{eq:sgd}
\end{eqnarray}
Here, $W=(w_1,\ldots,w_m)\in \mathbb R^{m \times d}$ is a matrix of learnable parameters (one per hidden neuron), and $s \in \mathbb R$ is a learnable offset. The output weights vector is fixed to $z = 1_m$, while $W$ and $c$ are optimized via SGD, where each update is on a single new sample point.

It is shown in Thm. 3 of \cite{Ghorbani19} that if $W_t \in \mathbb R^{m \times d}$ is the matrix of hidden parameters after $t$ steps of SGD, then in the limit \eqref{eq:proportionate}, the matrix $W_tW_t^\top \in \mathbb R^{m \times m}$ converges almost-surely to the best rank-$m$ approximation of $B$. Here, the almost-sure convergence is w.r.t randomness in the training data and in the initial random state of the parameters of the hidden neurons.  Thus, by continuity of matrix norms, we deduce that $\|W_tW_t^\top\|^2_F$ converges a.s to % $\sum_{k=1}^{m \land d}\lambda_k(B)^2 =:
$\|B\|_{F,m}^2$, in the infinite data limit $t \to \infty$.

Combining with Lemma \ref{lm:formula} establishes the following asymptotic formula for the (normalized) robustness of the resulting model $f_\sgd$, in the high-dimensional limit \eqref{eq:proportionate}.
% Combining with Lemma \ref{lm:formula}, we obtain the following result.

% \begin{mdframed}[backgroundcolor=cyan!10,rightline=false,leftline=false]
\begin{restatable}{thm}{}
In the limit \eqref{eq:proportionate}, it holds that
% \begin{eqnarray}
$\erob(f_\sgd)
% := \frac{\mathfrak{S}(f_\sgd)^2}{\mathfrak{S}(f_\star)^2}
\overset{a.s}{\longrightarrow} \dfrac{\|B\|_{F,m}^2}{\|B\|_F^2} = \dfrac{\sum_{k=1}^{m \land d}\lambda_k(B)^2}{\sum_{j=1}^d \lambda_k(B)^2} \le 1,$
% \end{eqnarray}
with equality iff $\rank(B) \le m$.
In particular, if $\rho \ge 1$, then in the limit \eqref{eq:proportionate}, it holds that
$\erob(f_\sgd) \overset{a.s}{\longrightarrow} 1$.
\label{thm:sobsgd}
\end{restatable}
% \end{mdframed}
\vspace{-.5cm}
\paragraph{Tradeoff approximation and robustness.} We see from the above theorem that the robustness measure for the SGD model converges to that of the true model if $m \ge d$, namely $\mathfrak{S}(f_\sgd)^2 \to \sum_{j=1}^d \lambda_j(B)^2 = \mathfrak{S}(f_\star)^2$ if $m \ge \rank(B)$ (for example, if $m \ge d$). Otherwise, if $m \le \rank(B)$, then the limiting value of $\mathfrak{S}(f_\sgd)^2$ can be arbitrarily less than $\mathfrak{S}(f_\star)^2$, i.e., the SGD model will be much more robust (i.e., stable) than the ground truth model. Comparing with Thm. 3 and Prop. 1 of \cite{Ghorbani19}, we can see that this increase in robustness is at the expense of increased test error. Indeed, it was shown in the aforementioned paper that
the normalized test error $\egen(f_\sgd)$ verifies
\begin{eqnarray}
\label{eq:sgdgenerr}
\egen(f_\sgd) \overset{p}{\longrightarrow} 1-\frac{\|B\|_{F,m}^2}{\|B\|_F^2},\text{ in the limit }\eqref{eq:proportionate}
\end{eqnarray}
Combining with our Thm. \ref{thm:sobsgd}, we deduce that
\begin{eqnarray}
\egen(f_\sgd) + \erob(f_\sgd) \overset{p}{\longrightarrow} 1,\text{ in the limit }\eqref{eq:proportionate}.
\label{eq:nnnfl}
\end{eqnarray}
The above formula highlights a tradeoff between test error and robustness. Thus, we have identified a novel tradeoff between approximation (test error) and robustness for the neural network model \eqref{eq:nn} trained via SGD. In the sequel, we shall establish such tradeoffs for other learning regimes.
% \mytodo{[Elvis] Analyze the under-parametrized case $m < d$.}

% \mytodo{Cite order papers which have identified tradeoff phenomena, e.g papers by Tsipras et al.; Kamalika et al; and Amin Karbasi et al.}

\section{Results for the random features (RF) model}
\label{sec:rf}
% \subsection{Preliminaries}
Consider the two-layer model \eqref{eq:nn} with hidden neuron parameters $w_1,\ldots,w_m$ sampled iid from a $d$-dimensional multivariate Gaussian distribution $N(0,\Gamma)$ with covariance matrix $\Gamma$. We denote this model $f_\rf$, which is thus given by
\begin{eqnarray}
\label{eq:rf}
f_\rf(x) = f_{W,z_\rf}(x)= z_\rf^\top \sigma(Wx),
\end{eqnarray}
where $z_\rf \in \mathbb R^m$ solves the following linear regression problem
\begin{eqnarray}
\arg\min_{z \in \mathbb R^m}\mathbb E_{x \sim N(0,I_d)}[(z^\top \sigma(Wx)-f_\star(x))^2].
% = \min_{a \in \mathbb R^m}\mathbb E_{x \sim N(0,I_d)}(f_\star(x)-f_\rf)^2.
\end{eqnarray}
 The covariance matrix $\Gamma$ encompasses the inductive bias of the neurons at initialization to different directions in feature space. Define the alignment $\alpha=\alpha(B,\Gamma)$ of the hidden neurons to the task at hand, namely learning the function $f_\star$, as follows
\begin{eqnarray}
\label{eq:alignment}
\alpha := \frac{\trace(B\Gamma)}{\|B\|_F\|\Gamma\|_F} \le 1.
\end{eqnarray}
% where $B$ is the coefficient matrix in the quadratic ground-truth model $f_\star$ given in \eqref{eq:truemodel}.
As we shall see, the task-alignment $\alpha$ plays a crucial role in the prediction performance (test error) and robustness dynamics of the resulting model, in the learning regimes considered in our work.
% \abcomment{this is the first time we introduce~$\Gamma$, which plays an important role throughout the paper. Maybe comment on its role, and the relationship to changing the covariance of input data.}

\subsection{Assumptions and key quantities}
\label{subsec:keyassume}
As in \cite{Ghorbani19}, we will need the following technical conditions in the current section and subsection sections of this article.
\begin{restatable}{cond}{}
The covariance matrix $\Gamma$ of the hidden neurons satisfies the following: (A) $\trace(\Gamma) = 1$ and $d\cdot\|\Gamma\|_{op} = \mathcal O(1)$.
(B) The empirical eigenvalue distribution of $d \cdot \Gamma$ converges weakly to a probability distribution $\mathcal D$ on $\mathbb R_+$.
\label{cond:traces} 
\end{restatable}
This condition is quite reasonable, and moreover, it allows us to leverage standard tools from random matrix theory (RMT) in our analysis. We will also need the following condition.
\begin{restatable}{cond}{}
The activation function $\sigma$ is weakly continuously-differentiable and satisfies the growth condition $\sigma(t)^2 \le c_0e^{c_1 t^2}$ for some $c_0>0$ and $c_1 < 1$, and for all $t \in \mathbb R$. Moreover, $\sigma$ is not a purely affine function.
\label{cond:growth}
\end{restatable}
The above growth condition is a classical condition usually imposed for the theoretical analysis of neural networks~\citep[see, e.g.,][]{Ghorbani19,Mei2019,montanari2020}, and is satisfied by all popular activation functions used in practice. One of its main purposes is to ensure that all the Hermite coefficients of the activation function exist.
Finally, we will assume the following.

\begin{restatable}{cond}{}
(A) $\lambda_0:=\lambda_0(\sigma) = 0$. (B) $\lambda_2:=\lambda_2(\sigma) \ne 0$.
\label{cond:nonlinear}
\end{restatable}
Part (A) of this condition was introduced by~\cite{Ghorbani19} to simplify the analysis of the test error of the random features model $f_\rf$. Part (B) ensures that $f_\rf$ does not degenerate to the null predictor.

% \mytodo{This condition is strictly not required.}
% The following result is an extension of Theorem \ref{thm:quadrfratio} to more general activation functions. It is one of our main contributions.

% \lorem 

% \begin{restatable}{cond}{}
% The alignment $\alpha$ defined in \eqref{eq:alignment} is lower-bounded in the limit $d \to \infty$.
% \end{restatable}
% \abdelete{\subsection{Key scalars and matrices}}

\begin{restatable}{df}{keyconstants}
Define the following scalars.
\begin{eqnarray}
\begin{split}
    \overline{\lambda} &:= \|\sigma\|^2 - \lambda_1^2,\,
    \kappa := \lambda_2^2\|\Gamma\|_F^2 d/2,\,
    \tau := \lambda_2\trace(B\Gamma)\sqrt{d},\\
    \overline{\lambda'} &:= \|\sigma'\|^2_{L^2(N(0,I_d))} - \lambda_1^2,\,
    \kappa' := \lambda_3^2\|\Gamma\|_F^2d/2.
\end{split}
\label{eq:allcoefs}
\end{eqnarray}
\end{restatable}
These coefficients will turn out to be “sufficient statistics” which completely capture the influence of activation function $\sigma$ on the robustness of the random features model $f_\rf$.  Note that by construction, $\overline{\lambda}$, $\kappa$, $\overline{\lambda'}$, and $\kappa'$ are nonnegative.
It is easily seen that for $x\sim N(0,I_d)$, and ground-truth model $f_\star$ defined in \eqref{eq:truemodel},
% \begin{eqnarray}
$z_\rf = U^{-1} v$,
%\label{eq:fittedarf}
%\end{eqnarray}
% where $U$ the random $m \times m$ matrix and the random $m$-dimensional vector $v$ are defined by % \abcomment{explain why~$U$ is always invertible, does this require assumptions on~$\Gamma$?}
\begin{align}
U_{jk} &:= \mathbb E_{x}[\sigma(x^\top w_j)\sigma(x^\top w_k)]\label{eq:randomU}\,\forall j,k \in [m],\\
v_j &:= \mathbb E_{x }[f_\star(x)\sigma(x^\top w_j)]\label{eq:randomv}\,\forall j \in [m].
\end{align}
%\abdelete{Note that $U$ is invertible thanks to Lm. 2 of the same paper.(mieux vaux rien mettre que pointer vers un lemme obscure dans le main paper?)}
% Note that the output weights vector for the random features model $f_\rf$ is given by
% The test error was computed in \cite{Ghorbani19}. Our objective in this section is to compute the robustness of $f_\rf$.
% \subsection{Main results for RF}

Now, consider the random psd matrices
% $A_0$ and $D_0$ defined by
\begin{eqnarray}
\begin{split}
A_0 &:= \overline{\lambda}I_m+\lambda_1^2 \Theta,\,
D_0 := \overline{\lambda'}I_m + (\kappa'/d + \lambda_1^2)\Theta,
\end{split}
\label{eq:A0D0}
\end{eqnarray}
with $\Theta:=WW^\top \in \mathbb R^{m \times m}$.
By standard RMT arguments, one can show that in the limit \eqref{eq:proportionate}, the random scalars $\trace(A_0^{-1})/d$ and $\trace(A_0^{-2}D_0)/d$ converge a.s to deterministic values. Denote these limiting values by $\psi_1$ and $\psi_2$, i.e.,
\begin{eqnarray}
%\begin{split}
\psi_1 := \lim_{\substack{m,d \to \infty\\d/m \to \rho}}\frac{\trace(A_0^{-1})}{d},\,
\psi_2 := \lim_{\substack{m,d \to \infty\\d/m \to \rho}}\frac{\trace(A_0^{-2}D_0)}{d}.
%\end{split}
\label{eq:psis}
\end{eqnarray}
Also, since $\overline{\lambda}$ and $\overline{\lambda'}$ are strictly positive (because $\sigma$ is not purely affine, by Condition \ref{cond:growth}), so are $\psi_1$ and $\psi_2$.
Moreover, using the so-called \emph{Silverstein fixed-point equation} from random matrix theory \citep{silverstein95,Ledoit2011,Dobriban2018}, it can be shown that $\psi_1$ and $\psi_2$ only depend on (i) the aspect ratio $\rho$, and (ii) the limiting eigenvalue distribution $\mathcal D$ of the rescaled covariance matrix $d\cdot \Gamma$ of the hidden neurons at initialization\footnote{For example, see \cite[Lemma 6]{Ghorbani19} for the case of $\psi_1$.}.
% \mytodo{Recall the exact fixed-point equations here!}
The scalars defined in \eqref{eq:allcoefs}, together with $\psi_{1,2}$ will play a crucial role in our analysis.

\subsection{Test error / prediction performance in RF regime}
We recall that the (normalized) test error $\egen(f_\rf)$ of the random features model $f_\rf$ was completely analyzed in \cite{Ghorbani19}. Indeed, it was established in Theorem 1 of \cite{Ghorbani19} that the following approximation holds
% the normalized test error verifies
\begin{eqnarray}
\begin{split}
\egen(f_\rf)
% &:= \frac{\varepsilon_{\mathrm{gen}}(f_\rf)}{\|f_\star\|_{L^2(N(0,I_d))}^2}\\
&= 1-\frac{\psi_1\tau^2}{\|B\|_F^2(2\kappa\psi_1 + 2)} + o_{d,\mathbb P}(1),\text{ in the limit \eqref{eq:proportionate}}.
\end{split}
\label{eq:rfgenerr}
\end{eqnarray}

Thus, the (normalized) test error only depends on the aspect ratio $\rho$, the limiting spectral distribution $\mathcal D$ of $d\cdot \Gamma$, and the scale parameters $(\overline{\lambda},\kappa,\tau)$ defined in \eqref{eq:allcoefs}. 
% where $\psi$ is the unique positive solution of the Silverstein fixed-point equation.
It was further shown in \cite{Ghorbani19} that, if the task-alignment $\alpha$ of the hidden neurons defined in \eqref{eq:alignment}, admits a limit $\alpha_\infty$ when $d \to \infty$, then w.p 1
\begin{eqnarray}
\label{eq:limitingenerror}
\lim_{\rho \to \infty}\lim_{\substack{m,d \to \infty\\m/d \to \rho}}  \egen(f_\rf) = 1-\alpha_\infty^2.
\end{eqnarray}
Thus, \eqref{eq:rfgenerr} predicts that the (normalized) test error $\egen(f_\rf)$ vanishes if (i) $\Gamma \propto B$ and (ii) the number of neurons per input dimension $m/d$ diverges.

\begin{figure}[t]
    \centering
    \includegraphics[width=.9\linewidth]{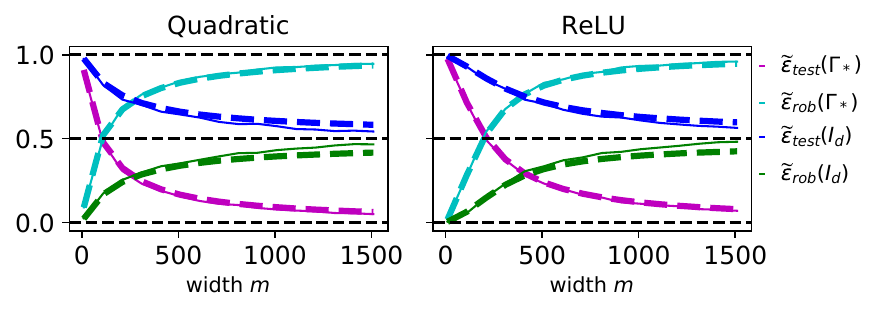}
    % {rf_QuadraticActivationFunction_d=450_n=100000.png}
    \vspace{-.4cm}
    \caption{Empirical validation of Theorem \ref{thm:rfratio} and Corollary \ref{cor:rfratio}. Showing (normalized) test / test error $\egen$ and robustness $\erob$ of random features model \eqref{eq:rf} as a function of the number of the network width, for different choices of the covariance matrix $\Gamma$ of the random weights of the hidden neurons: the optimal choice $\Gamma_\star \propto B$ and the naive choice $I_d$. Here, the input-dimension is $d=450$ and the regularization $\lambda$ is zero. Horizontal broken lines correspond to asymptotes at $\alpha_\infty^2$ at $1-\alpha_\infty^2$, where $\alpha_\infty := \lim_{d \to \infty}\trace(B\Gamma)/(\|B\|_F\|\Gamma\|_F)$ is the level of task-alignment of the covariance matrix $\Gamma$ of hidden neurons, w.r.t learning the ground-truth function $f_\star$ defined in \eqref{eq:truemodel}. Broken curves are theoretical predictions, while solid curves correspond to actual experiments.
    % \mytodo{Include a subfig for ReLU activation too!}
    % \mytodo{Add error bars}. \mytodo{Change $\widetilde{s}$ to $\erob$.}
    }
    \label{fig:rf}
\end{figure}
\subsection{Analysis of robustness in RF regime}
The following result establishes an analytic formula for the robustness in the RF regime.
% \begin{mdframed}[backgroundcolor=cyan!10,rightline=false,leftline=false]
\begin{restatable}{thm}{rfratio}
% Let the covariance matrix of the data be $\Lambda \propto I_d$ and suppose Conditon \ref{cond:traces} is in order.
Consider the random features model $f_\rf$ \eqref{eq:rf}, with covariance matrix $\Gamma$ satisfying Condition \ref{cond:traces} and activation function $\sigma$ satisfying Conditions \ref{cond:growth} and \ref{cond:nonlinear}.
% Let $\lambda \ge 0$ be a ridge-regularization parameter in the problem of estimating the output weights $a$ of $f_\rf$. 

(A) In the limit \eqref{eq:proportionate}, we have the following approximation
\begin{eqnarray}
% \begin{split}
\erob(f_\rf)
% &:= \dfrac{\mathfrak{S}(f_\rf)^2}{\mathfrak{S}(f_\star)^2}\\
=
\frac{\tau^2(2\kappa \psi_1^2 + \psi_2)}{\|B\|^2_F(2\kappa \psi_1+2)^2} + o_{d,\mathbb P}(1).
% \beta^2 + (\alpha-\beta)^2 \cdot \kappa \trace(A_0^{-2}D_0)/d + o_{d,\mathbb P}(1),
%\end{split}
\label{eq:rfratio}
\end{eqnarray}
% where the (random) scalar $\beta$ is defined by
% \begin{eqnarray}
% \beta = \beta(B,\Gamma) := (1-\widetilde{\varepsilon}_{{\rm gen}}(f_\rf))/\alpha,
% \end{eqnarray}
% and $\alpha$ is the "alignment" defined in \eqref{eq:alignment}.
% where $\varepsilon_{\mathrm{gen}}(f_\rf) := \mathbb E_{x \sim N(0,I_d)}[(f_\rf(x)-f_\star(x))^2]$ is the test error of the RF model, and 
(B) If the alignment $\alpha$ admits a limit $\alpha_\infty$ as $d \to \infty$, then w.p $1$ it holds that
\begin{eqnarray}
\lim_{\rho \to \infty}\lim\limits_{\substack{m,d \to \infty\\ m/d \to \rho}}\erob(f_\rf) = \alpha_\infty^2 \le 1.
\end{eqnarray}
In particular, for the optimal choice of $\Gamma$ in terms of test error, namely $\Gamma \propto B$, it holds w.p.~$1$ that
\begin{eqnarray}
\lim_{\rho \to \infty}\lim\limits_{\substack{m,d \to \infty\\ m/d \to \rho}}\erob(f_\rf) = 1.
\end{eqnarray}
\label{thm:rfratio}
\end{restatable}
\vspace{-.5cm}
Thus, the robustness only depends on the aspect ratio $\rho$, the limiting spectral distribution $\mathcal D$ of $d\cdot \Gamma$ (via $\psi_1$ and $\psi_2$), and the scale parameters defined in \eqref{eq:allcoefs}. The theorem is proved in the Appendix~\ref{subsec:rf_proof}. 
% \begin{restatable}{rmk}{}

\paragraph{Tradeoff between approximation and robustness.}
We deduce from the above theorem that in the regime \eqref{eq:proportionate}, the random features model $f_\rf$ is more robust (i.e., stable) than the ground-truth model $f_\star$.
%This is understood by the fact that the inductive bias of the former is towards linear models.
Interestingly, we see that this gap in robustness between the two models closes with increasing alignment $\alpha$ between the covariance matrix of the random features $\Gamma$ and the ground-truth matrix $B$.
% This is unlike the test error of $f_\rf$ which increases with increased alignment between $\Gamma$ and $B$, as can be see from \eqref{eq:limitingenerror}.
Comparing with \eqref{eq:limitingenerror}, we obtain the following relationship,
\begin{eqnarray}
\egen(f_\rf) + \erob(f_\rf) \to 1,\text{ in the limit }\eqref{eq:proportionate},
\label{eq:rfnfl}
\end{eqnarray}
which trades-off between the normalized test error  $\widetilde{\varepsilon}_{{\rm test}}(f_\rf)$ (defined in \eqref{eq:rfgenerr}) and the normalized robustness $\erob(f_\rf)$ of $f_\rf$.
Thus, we have identified another novel trade-off between the test error and the robustness in random features models. % in the extremely over-parametrized ($\rho \to \infty$) regime with infinite training data.

% \lorem

% \begin{figure}
%     \centering
%     \includegraphics[width=.7\linewidth]{rf_sob2.png}
%     % \includegraphics[width=.48\linewidth]{rf_err.png}
%     \caption{Empirical validation of results for RF model with quadratic activation function. For this experiment, the input-dimension is $d=200$. Notice how the normalized test error and the normalized nonrobustness as mirror reflections of one another about the horizontal line at 0.5, reminiscent of a generalization / robustness compromise, and in accordance to \eqref{eq:rfnfl}.}
%     \label{fig:my_label}
% \end{figure}

\paragraph{The case of quadratic activations.}
We now specialize Theorem \ref{thm:rfratio} to the case of the quadratic activation function and obtain more transparent formulae.
% \begin{mdframed}[backgroundcolor=cyan!10,rightline=false,leftline=false]
\begin{restatable}{cor}{corrfratio}
Consider the random features model $f_\rf$ with covariance matrix $\Gamma$ satisfying Condition \ref{cond:traces} and quadratic activation function $\sigma(t) := t^2 - 1$. Then, in the limit \eqref{eq:proportionate}, it holds that
% \begin{eqnarray}
$\erob(f_\rf) = \dfrac{\trace(B\Gamma)^2\|\Gamma\|_F^2}{(2/m+\|\Gamma\|_F^2)^2\|B\|_F^2} + o_{d,\mathbb P}(1).$
% \end{eqnarray}
Furthermore, parts (B) and (C) of Theorem \ref{thm:rfratio} hold.
\label{cor:rfratio}
\end{restatable}
% \end{mdframed}
Theorem~\ref{thm:rfratio} and Corollary ~\ref{cor:rfratio} are empirically verified in Fig.~\ref{fig:rf}. Notice the perfect match between our theoretical results and experiments.
\section{Neural tangent (NT) regime}
\label{sec:nt}
Consider a two-layer network with output weights $\ainit_j$ fixed to $1$, and hidden weight $w_j \in \mathbb R^d$ drawn from $N(0,\Gamma)$. For the quadratic activation $\sigma(t) := t^2$, the neural tangent (NT) approximation \cite{jacot18,chizat2018lazy} w.r.t.~the first layer parameters is given by
\vspace{-.3cm}
\begin{eqnarray}
f_{W+A}(x) \approx f_\init(x) + \trace(A\nabla_W f_W(x))
= f_\init(x) + 2\sum_{j=1}^m (x^\top (\ainit_ja_j))(x^\top w_j).
\label{eq:ntcalc}
% \vspace{-.3cm}
\end{eqnarray}
where $f_\init$ is the function computed by the neural network at initialization (see Appendix \ref{sec:init} for details), and $A = (a_1,\ldots,a_m) = (\Delta w_1,\ldots, \Delta w_m) \in \mathbb R^{m \times d}$ is the change in $W$. We will see that the initialization term $f_\init$ might have drastic influence on the robustness of the resulting model. % in NT regime.

% \abedit{Although in practice deep neural networks may be far from such linearized regimes~\citep{chizat2018lazy}, there is empirical evidence that in some cases some of the weights may remain close to initialization~\citep[see, e.g.,][]{zhang2019all}. This suggests that understanding generalization and robustness properties in this regime may still provide some insight on what may happen in practice, particularly regarding the impact of initialization.}

\subsection{NT approximation without initialization term}
We temporarily discard the initialization term $f_\init(x)$ from the RHS of \eqref{eq:ntcalc}, and consider the simplified approximation %  $f_\nt$ to be the remaining term, i.e
\vspace{-.4cm}
\begin{eqnarray}
f_\nt(x;A;c) := 2\sum_{j=1}^m (x^\top a_j)(x^\top w_j)-c,
\label{eq:nt}
\end{eqnarray}
where, without loss of generality, we absorb the output weights $z_j$ in the parameters~$a_j$ in the first-order term.
% Later, we will consider the so-called lazy training regime were the $f_\init(x)$ is not discarded.
In~\eqref{eq:nt},
$A \in \mathbb R^{m \times d}$ and $c \in \mathbb R$ are model parameters that are optimized. 
In terms of test error, let $A_\nt$ and $c_\nt$ be optimal in $f_\nt(\cdot;A,c)$, and let $f_\nt = f_\nt(\cdot;A_\nt,c_\nt)$ for short.
% where $w_1,\ldots,w_m \sim \mathbb R^d$ are random iid from $\mathcal N(0,(1/d)I_d)$ (as in the RF model), and $a_1,\ldots,a_m \in \mathbb R^d$ are optimized on training data (the entire data distribution).
In Thm. 2 of \cite{Ghorbani19}, it is shown that the (normalized) test error of the linearized model $f_\nt$ is given by
\begin{eqnarray}
\label{eq:generrornt}
% \begin{split}
\mathbb E_W [\egen(f_\nt)] 
%:= \frac{\varepsilon_{\mathrm{gen}}(f_\nt)}{\|f_\star\|_{L^2}}
&= (1-\rho)_+^2(1-\beta) + (1-\rho)_+\beta + o_d(1).
% \end{split}
\end{eqnarray}
where $\beta = \beta(B) := \trace(B)^2/(d\|B\|_F^2) \in [0,1]$.
Our objective in this section is to compute the robustness of $f_\nt$. Let $\rho := \min(\rho,1)$.
% \subsection{The (non)robustness of $f_\nt$}
% Let $W=(w_1,\ldots,w_m) \in \mathbb R^{m \times d}$ and $A = (a_1,\ldots,a_m) \in \mathbb R^{m \times d}$.

% We start with the following formula.
\begin{restatable}{lm}{sobnt}
 $\mathfrak{S}(f_\nt)^2 = 4\|W^\top A + A^\top W\|_F^2.
$
% \begin{eqnarray}
% \mathfrak{S}(f_\nt)^2 = \frac{1}{2}\sum_{j,k} (w_j^\top w_k)(a_j^\top \Lambda a_k) + 2(w_j^\top a_k)(a_j^\top \Lambda w_k) +(a_j^\top a_k)(w_j^\top \Lambda w_k),
% \label{eq:isotropicNTK}
% \end{eqnarray}
% where as usual, $\Lambda \in \mathbb R^{d \times d}$ is the covariance matrix of the data distribution $\mu$.

% In particular, If $\Lambda = \sigma^2 I_d$, then
% \begin{eqnarray}
% \end{eqnarray}
\label{lm:sobnt}
\end{restatable}

% \subsection{NT in over-parametrized regime ($\rho \ge 1$)}
% Temporarily, let's assume isotropic design matrix $\Lambda = I_d$, so that $\mathfrak{S}(f_\nt)^2 = 4 (\|A^\top W\|_F^2 + \trace((WA^\top)^2)))$, thanks to Lemma \ref{lm:sobnt}.
% which only depends o$m \times m$ matrices $WW^\top$, $AA^\top$, and $WA^\top$.
Combining with Lemma \ref{lm:sobnt} above, we can prove the following result (see appendix).
% observe that for such a choice of $A$, it holds that
% \begin{eqnarray}
% \begin{split}
% \mathfrak{S}(f_\nt)^2 &= 4\|W^\top A + A^\top W\|_F^2\\
% &= 4\|P_1P_1^\top B/2+BP_1P_1^\top/2\|_F^2 \\
% &= 2\|P_1P_1^\top B/2\|_F^2 + 2\|P_1BP_1^\top\|_F^2.
% \end{split}
% \label{eq:ntwizard}
% \end{eqnarray}
% Further, one can show that (see Lemma xyz)
% \begin{eqnarray}
% \begin{split}
%     \mathbb E_W \|P_1P_1^\top B\|_F^2 &= \|B\|_F^2(1+\underline{\rho} + o_d(1)),\\
%     \mathbb E_W \|P_1BP_1^\top\|_F^2 &= \|B\|_F^2(\underline{\rho}^2(1-\beta) + \underline{\rho}\beta + o_d(1)),
% \end{split}
% \end{eqnarray}
% where $\underline{\rho} := \min(\rho,1)$.

% Taking expectations w.r.t $W$ and noting that
% \begin{eqnarray}
% \begin{split}
% \mathbb E_{W}[P_1P_1^\top] &= (\min(m,d)/d)I_d = \min(m/d,1)I_d\\
% &= (\min(\rho,1) + o_d(1))I_d,
% \end{split}
% \end{eqnarray}
% we deduce that
% \begin{eqnarray}
% \begin{split}
% \mathbb E_W[\mathfrak{S}(f_\nt)^2] &= 4\trace(\mathbb E_W[P_1P_1^\top]B^2)\\
% &= 4\|B\|_F^2\cdot(\min(\rho,1) + o_d(1))\\
% &= \mathfrak{S}(f_\star)^2 \cdot(\min(\rho,1) + o_d(1)).
% \end{split}
% \end{eqnarray}
% Thus, we have proved the following theorem which is one of our main results.

\begin{restatable}{thm}{purent}
  \label{thm:purent}
Consider the neural tangent model $f_\nt$ in \eqref{eq:nt}. In the limit \eqref{eq:proportionate}
% % with $\rho \ge 1$ (i.e over-parametrized regime),
it holds that,
\begin{eqnarray}
\label{eq:ntsobsob}
% \begin{split}
\mathbb E_W [\erob(f_\nt)]
% := \frac{\mathbb E_W[\mathfrak{S}(f_\nt)^2]}{\mathfrak{S}(f_\star)^2}
= \frac{\underline{\rho} + \underline{\rho}^2}{2}+\frac{(\underline{\rho}-\underline{\rho}^2)\beta}{2} + o_{d}(1).
% \\ &\le \underline{\rho}+o_d(1) \le 1 + o_d(1),
% \end{split}
\end{eqnarray}
% Thus, in the over-parametrized regime, the neural tangent model $f_\nt$ is as robust as the one trained via SGD $f_\sgd$, which is itself as robust as the underlying true model $f_\star$.
\end{restatable}
Further observe that because $\beta % :=\trace(B)^2/(d\|B\|_F^2)
\le 1$, the RHS of  \eqref{eq:ntsobsob} is further upper-bounded by $\underline{\rho} \le 1$ with equality when $\beta = 1$ (e.g., for $B \propto I_d$). We deduce that in the NT regime, the neural network is at least as robust as the ground-truth model $f_\star$.
% \mytodo{[Elvis] It seams the interesting case is the under-complete $\rho < 1$.}
Comparing with \eqref{eq:generrornt}, we obtain the following tradeoff between test error and robustness, stated only for~$\beta=1$ for simplicity of presentation.
\begin{restatable}{cor}{}
If $\beta = 1$ (i.e., if $B \propto I_d)$, then in the limit \eqref{eq:proportionate}, it holds that
\begin{eqnarray}
\mathbb E_W[\egen(f_\nt) + \erob(f_\nt)] &= 1 + o_d(1).
\end{eqnarray}
\end{restatable}
% Note that the above result is stated only for $\beta=1$ for simplicity of presentation.
\begin{figure}[!]
    \centering
        \hspace{-0.15in}
\subfigure[NN at large random init]{\includegraphics[width=0.26\textwidth]{{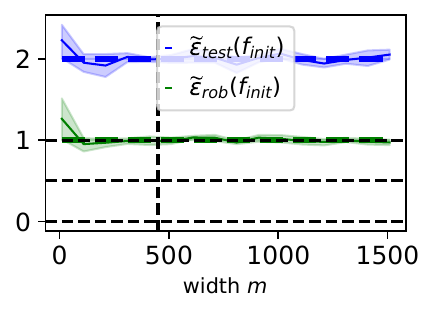}}\label{fig:init}}
        \hspace{-0.1in}
        \subfigure[NT regime]{\includegraphics[width=0.25\textwidth]{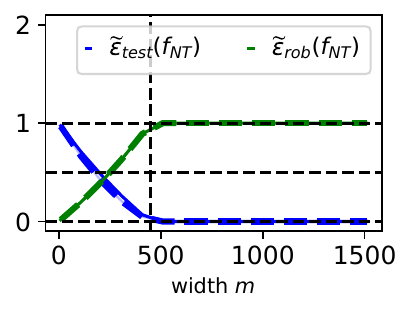}\label{fig:nt}}
        \hspace{-0.1in}
\subfigure[Lazy, small random init]{\includegraphics[width=0.25\textwidth]{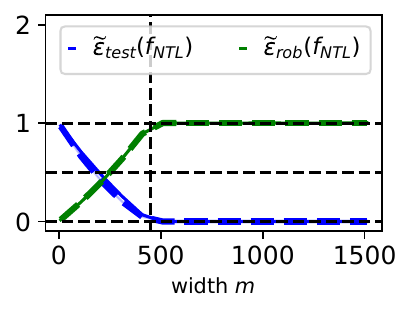}}
        \hspace{-0.1in}
\subfigure[Lazy, large random init]{\includegraphics[width=0.25\textwidth]{{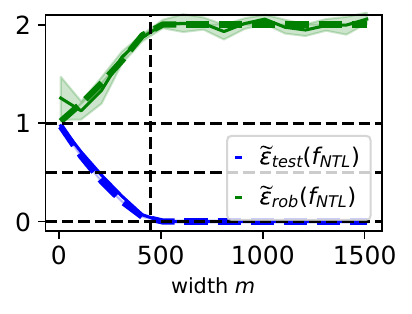}}\label{fig:ntllarge}}
    \caption{Showing curves of (normalized) test error $\egen$ and robustness $\erob$ for a two-layer neural network in different learning regimes of the hidden weights. Here, the input dimension is $d=450$ and the width $m$ sweeps a range of values from $10$ to $1500$. Dashed curves correspond to theoretical predictions, while solid curves correspond to actual values observed in the experiment (5 runs). We use $n=10^6$ training samples as a proxy for infinite data.
    % The test error and robustness of the model is equivalent to that of the null-predictor.
    % Notice the tradeoff between test error and robustness, illustrated by the fact that $\egen(f_\sgd) + \erob(f_\sgd) \approx 1$.
    The covariance matrix of the hidden neurons is fixed at $\Gamma = (1/d)I_d$. For simplicity, simplicity of this experiment, we also take the coefficient matrix in the ground-truth function as $B \propto I_d$. \textbf{(c)} "Small random init" means $B=(1/\sqrt{d})I_d$, so that $B$ is much larger than $\Gamma$.
    \textbf{(d)} "Large random init" means that
    % sthe coefficient matrix is
    $B=(1/d)I_d$, and thus is of the same order as $\Gamma$ (in Frobenius norm).
    In this case, the initialization degrades the robustness, as predicted by Thm. \ref{thm:ntl}. Note that, as predicted by Thm. \ref{thm:ntl}, random initialization has no impact on the test error of the neural tangent approximation. Results for \textbf{(a)} the neural network at initialization (Thm. \ref{thm:untrained} and \ref{thm:untrainedgenerr}) and \textbf{(b)} In the neural tangent regime (Thm. \ref{thm:purent}) are also depicted for reference.
    % Large random init means $B=$
    % \textbf{(b)} At initialization (random hidden and output weights), with $B=I_d$ and $\Gamma=(1/d)I_d$. \textbf{Right.} (Low SNR) $B_d = (1/\sqrt{d})I_d$, $\Gamma = (1/d)I_d$. But the test error and robustness of the model are bad.
    }
    \label{fig:allnt}
\end{figure}
\vspace{-.3cm}
\subsection{NT approximation with initialization term}
We now consider the neural tangent approximation \eqref{eq:ntcalc} without discarding the initialization term $f_\init$ from the RHS of \eqref{eq:ntcalc}. Also, let $\ainit \in \mathbb R^m$ be the output weights, drawn iid from $N(0,1/m)$ and frozen, and let $Q$ be the $m \times m$ diagonal matrix with $\ainit$ as its diagonal. This corresponds to what could be referred to as \emph{lazy training} regime of the hidden layer.
Let $f_\ntl(x;A,c)$ be RHS of \eqref{eq:ntcalc},
\begin{eqnarray}
f_\ntl(x;A,c) &:= f_{W,z,c}(x) + 2\sum_{j=1}^m (x^\top w_j)(x^\top a_j) =f_\init(x) + f_\nt(x;A,c),
\end{eqnarray}
where $f_\init(x) := \sum_{j=1}^m \ainit_j(x^\top w_j)^2 = x^\top W^\top Q Wx$ defines the neural network at initialization.
% , and we absorb the output weights in the parameters~$a_j$ in the first-order term.

% The following theorem is one of our main contributions.
\begin{restatable}{thm}{ntL}
\label{thm:ntl}
Suppose the output weights $\ainit$ at initialization are iid from $N(0,(1/m)I_m)$. Then, in the limit \eqref{eq:proportionate}, the following identities hold
\begin{align}
\mathbb E_{\{W,\ainit\}}[\egen(f_\ntl)] &= \mathbb E_W[\egen(f_\nt)] + o_d(1)\label{eq:ntlgen},\\
\mathbb E_{\{W,\ainit\}} [\erob(f_\ntl)] & = \mathbb E_W[\erob(f_\nt)] +\mathbb E_{\{W,\ainit\}}[\erob(f_\init)] + o_d(1)\label{eq:ntlsob}.
\end{align}
\end{restatable}
Thus, on average (over initialization):
(i) $f_\ntl$ and $f_\nt$ have the same test error, i.e., the initialization term $f_\init$ in $f_\ntl$ does not affect its test error.
(ii) On the other hand, $f_\ntl$ is less robust than $f_\nt$; the deficit in robustness, namely the term $(1+\|\Gamma\|^2)/\|B\|_F^2$, corresponds exactly to the contribution of the initialization.
The situation is empirically illustrated in Fig. \ref{fig:allnt}. Notice the perfect match between our theoretical results and experiments.

% \mytodo{Experiments for NT model!}

\section{Concluding remarks}
In this paper, we have studied the adversarial robustness of two-layer neural networks in different  high-dimensional learning regimes, and established a number of new tradeoffs between prediction performance and robustness. Our analysis also shows that random initialization can further degrade the robustness in lazy training regimes: for "large" random initialization, the trained neural network inherits an additional nonrobustness already present at initialization.

Our work can be seen as a first step towards a rigorous analysis of the robustness of trained neural networks, a subject which is still understudied.

\bibliography{literature}

\begin{thebibliography}{39}
\providecommand{\natexlab}[1]{#1}
\providecommand{\url}[1]{\texttt{#1}}
\expandafter\ifx\csname urlstyle\endcsname\relax
  \providecommand{\doi}[1]{doi: #1}\else
  \providecommand{\doi}{doi: \begingroup \urlstyle{rm}\Url}\fi

\bibitem[Ali et~al.(2020)Ali, Dobriban, and
  Tibshirani]{implicitregTibshirani2020}
Ali, A., Dobriban, E., and Tibshirani, R.~J.
\newblock The implicit regularization of stochastic gradient flow for least
  squares.
\newblock In \emph{Proceedings of the 37th International Conference on Machine
  Learning, {ICML}}, volume 119 of \emph{Proceedings of Machine Learning
  Research}, pp.\  233--244. {PMLR}, 2020.

\bibitem[Bartlett et~al.(2021)Bartlett, Bubeck, and
  Cherapanamjeri]{bartlett2021}
Bartlett, P.~L., Bubeck, S., and Cherapanamjeri, Y.
\newblock Adversarial examples in multi-layer random relu networks.
\newblock \emph{NeurIPS}, 2021.

\bibitem[Bhattacharjee et~al.(2021)Bhattacharjee, Jha, and
  Chaudhuri]{bhattacharjee2020sample}
Bhattacharjee, R., Jha, S., and Chaudhuri, K.
\newblock Sample complexity of robust linear classification on separated data.
\newblock In \emph{Proceedings of the 38th International Conference on Machine
  Learning (ICML)}. PMLR, 2021.

\bibitem[Bubeck \& Sellke(2021)Bubeck and Sellke]{bubeck2021universal}
Bubeck, S. and Sellke, M.
\newblock A universal law of robustness via isoperimetry.
\newblock In \emph{Advances in Neural Information Processing Systems}, 2021.

\bibitem[{Bubeck} et~al.(2020b){Bubeck}, {Li}, and {Nagaraj}]{lor}
{Bubeck}, S., {Li}, Y., and {Nagaraj}, D.
\newblock {A law of robustness for two-layers neural networks}.
\newblock \emph{arXiv e-prints}, art. arXiv:2009.14444, September 2020b.

\bibitem[Bubeck et~al.(2021)Bubeck, Cherapanamjeri, Gidel, and des
  Combes]{bubecksinglestep2021}
Bubeck, S., Cherapanamjeri, Y., Gidel, G., and des Combes, R.~T.
\newblock A single gradient step finds adversarial examples on random
  two-layers neural networks.
\newblock In \emph{Advances in Neural Information Processing Systems}, 2021.

\bibitem[Chizat et~al.(2019)Chizat, Oyallon, and Bach]{chizat2018lazy}
Chizat, L., Oyallon, E., and Bach, F.
\newblock On lazy training in differentiable programming.
\newblock 2019.

\bibitem[Daniely \& Shacham(2020)Daniely and Shacham]{mostrelu2020}
Daniely, A. and Shacham, H.
\newblock Most relu networks suffer from \textbackslash ell\^{}2 adversarial
  perturbations.
\newblock In \emph{Advances in Neural Information Processing Systems},
  volume~33, pp.\  6629--6636. Curran Associates, Inc., 2020.

\bibitem[Dobriban \& Wager(2018)Dobriban and Wager]{Dobriban2018}
Dobriban, E. and Wager, S.
\newblock {High-dimensional asymptotics of prediction: Ridge regression and
  classification}.
\newblock \emph{The Annals of Statistics}, 46\penalty0 (1):\penalty0 247 --
  279, 2018.

\bibitem[Dobriban et~al.(2020)Dobriban, Hassani, Hong, and
  Robey]{dobriban2020provable}
Dobriban, E., Hassani, H., Hong, D., and Robey, A.
\newblock Provable tradeoffs in adversarially robust classification.
\newblock \emph{arXiv preprint arXiv:2006.05161}, 2020.

\bibitem[Dohmatob(2019)]{dohmatob19}
Dohmatob, E.
\newblock Generalized no free lunch theorem for adversarial robustness.
\newblock In \emph{Proceedings of the 36th International Conference on Machine
  Learning (ICML)}, volume~97 of \emph{Proceedings of Machine Learning
  Research}. {PMLR}, 2019.

\bibitem[Dohmatob(2021)]{dohmatob2021fundamental}
Dohmatob, E.
\newblock Fundamental tradeoffs between memorization and robustness in random
  features and neural tangent regimes.
\newblock \emph{arXiv preprint arXiv:2106.02630}, 2021.

\bibitem[El~Karoui(2010)]{elkaroui2010}
El~Karoui, N.
\newblock The spectrum of kernel random matrices.
\newblock \emph{Ann. Statist.}, 2010.

\bibitem[Gao et~al.(2019)Gao, Cai, Li, Hsieh, Wang, and Lee]{RuiqiGao2019}
Gao, R., Cai, T., Li, H., Hsieh, C.-J., Wang, L., and Lee, J.~D.
\newblock Convergence of adversarial training in overparametrized neural
  networks.
\newblock In \emph{Advances in Neural Information Processing Systems},
  volume~32. Curran Associates, Inc., 2019.

\bibitem[Ghorbani et~al.(2019)Ghorbani, Mei, Misiakiewicz, and
  Montanari]{Ghorbani19}
Ghorbani, B., Mei, S., Misiakiewicz, T., and Montanari, A.
\newblock Limitations of lazy training of two-layers neural network.
\newblock In \emph{Advances in Neural Information Processing Systems},
  volume~32. Curran Associates, Inc., 2019.

\bibitem[Gilmer et~al.(2018)Gilmer, Metz, Faghri, Schoenholz, Raghu,
  Wattenberg, and Goodfellow]{gilmerspheres18}
Gilmer, J., Metz, L., Faghri, F., Schoenholz, S.~S., Raghu, M., Wattenberg, M.,
  and Goodfellow, I.~J.
\newblock Adversarial spheres.
\newblock \emph{CoRR}, abs/1801.02774, 2018.

\bibitem[Hassani \& Javanmard(2022)Hassani and Javanmard]{hassani2022curse}
Hassani, H. and Javanmard, A.
\newblock The curse of overparametrization in adversarial training: Precise
  analysis of robust generalization for random features regression.
\newblock \emph{arXiv preprint arXiv:2201.05149}, 2022.

\bibitem[Jacot et~al.(2018)Jacot, Gabriel, and Hongler]{jacot18}
Jacot, A., Gabriel, F., and Hongler, C.
\newblock Neural tangent kernel: Convergence and generalization in neural
  networks.
\newblock In \emph{Advances in Neural Information Processing Systems 31}. 2018.

\bibitem[Khim \& Loh(2018)Khim and Loh]{khim2018adversarial}
Khim, J. and Loh, P.-L.
\newblock Adversarial risk bounds via function transformation.
\newblock \emph{arXiv preprint arXiv:1810.09519}, 2018.

\bibitem[Ledoit \& P{\'e}ch{\'e}(2011)Ledoit and P{\'e}ch{\'e}]{Ledoit2011}
Ledoit, O. and P{\'e}ch{\'e}, S.
\newblock Eigenvectors of some large sample covariance matrix ensembles.
\newblock \emph{Probability Theory and Related Fields}, Oct 2011.

\bibitem[Madry et~al.(2018)Madry, Makelov, Schmidt, Tsipras, and
  Vladu]{madry2017towards}
Madry, A., Makelov, A., Schmidt, L., Tsipras, D., and Vladu, A.
\newblock Towards deep learning models resistant to adversarial attacks.
\newblock 2018.

\bibitem[Mahloujifar et~al.(2018)Mahloujifar, Diochnos, and
  Mahmoody]{saeed2018}
Mahloujifar, S., Diochnos, D.~I., and Mahmoody, M.
\newblock The curse of concentration in robust learning: Evasion and poisoning
  attacks from concentration of measure.
\newblock \emph{CoRR}, abs/1809.03063, 2018.

\bibitem[{Mei} \& {Montanari}(2019){Mei} and {Montanari}]{Mei2019}
{Mei}, S. and {Montanari}, A.
\newblock {The generalization error of random features regression: Precise
  asymptotics and double descent curve}.
\newblock \emph{arXiv e-prints}, art. arXiv:1908.05355, August 2019.

\bibitem[Min et~al.(2021{\natexlab{a}})Min, Chen, and Karbasi]{Karbasi2021}
Min, Y., Chen, L., and Karbasi, A.
\newblock The curious case of adversarially robust models: More data can help,
  double descend, or hurt generalization.
\newblock In \emph{Proceedings of the Thirty-Seventh Conference on Uncertainty
  in Artificial Intelligence}. PMLR, 2021{\natexlab{a}}.

\bibitem[Min et~al.(2021{\natexlab{b}})Min, Chen, and Karbasi]{min2021curious}
Min, Y., Chen, L., and Karbasi, A.
\newblock The curious case of adversarially robust models: More data can help,
  double descend, or hurt generalization.
\newblock In \emph{Uncertainty in Artificial Intelligence (UAI)},
  2021{\natexlab{b}}.

\bibitem[Montanari \& Zhong(2020)Montanari and Zhong]{montanari2020}
Montanari, A. and Zhong, Y.
\newblock The interpolation phase transition in neural networks: Memorization
  and generalization under lazy training.
\newblock \emph{CoRR}, abs/2007.12826, 2020.

\bibitem[Moosavi{-}Dezfooli et~al.(2017)Moosavi{-}Dezfooli, Fawzi, Fawzi,
  Frossard, and Soatto]{moosavi17}
Moosavi{-}Dezfooli, S., Fawzi, A., Fawzi, O., Frossard, P., and Soatto, S.
\newblock Analysis of universal adversarial perturbations.
\newblock abs/1705.09554, 2017.

\bibitem[Rahimi \& Recht(2008)Rahimi and Recht]{rf}
Rahimi, A. and Recht, B.
\newblock Uniform approximation of functions with random bases.
\newblock 2008.

\bibitem[Sanyal et~al.(2021)Sanyal, Dokania, Kanade, and
  Torr]{sanyal2021benign}
Sanyal, A., Dokania, P.~K., Kanade, V., and Torr, P.~H.
\newblock How benign is benign overfitting?
\newblock In \emph{International Conference on Learning Representations
  (ICLR)}, 2021.

\bibitem[Schmidt et~al.(2018)Schmidt, Santurkar, Tsipras, Talwar, and
  Madry]{schmidt2018}
Schmidt, L., Santurkar, S., Tsipras, D., Talwar, K., and Madry, A.
\newblock Adversarially robust generalization requires more data.
\newblock \emph{CoRR}, abs/1804.11285, 2018.

\bibitem[Shafahi et~al.(2018)Shafahi, Huang, Studer, Feizi, and
  Goldstein]{goldstein}
Shafahi, A., Huang, W.~R., Studer, C., Feizi, S., and Goldstein, T.
\newblock Are adversarial examples inevitable?
\newblock \emph{CoRR}, abs/1809.02104, 2018.

\bibitem[Silverstein \& Choi(1995)Silverstein and Choi]{silverstein95}
Silverstein, J. and Choi, S.
\newblock Analysis of the limiting spectral distribution of large dimensional
  random matrices.
\newblock \emph{Journal of Multivariate Analysis}, 1995.

\bibitem[Simon-Gabriel et~al.(2019)Simon-Gabriel, Ollivier, Bottou,
  Sch{\"o}lkopf, and Lopez-Paz]{yann2019firstorder}
Simon-Gabriel, C.-J., Ollivier, Y., Bottou, L., Sch{\"o}lkopf, B., and
  Lopez-Paz, D.
\newblock First-order adversarial vulnerability of neural networks and input
  dimension.
\newblock In \emph{Proceedings of the 36th International Conference on Machine
  Learning}, volume~97 of \emph{Proceedings of Machine Learning Research}, pp.\
   5809--5817. PMLR, 09--15 Jun 2019.

\bibitem[Szegedy et~al.(2013)Szegedy, Zaremba, Sutskever, Bruna, Erhan,
  Goodfellow, and Fergus]{szegedy2013intriguing}
Szegedy, C., Zaremba, W., Sutskever, I., Bruna, J., Erhan, D., Goodfellow, I.,
  and Fergus, R.
\newblock Intriguing properties of neural networks.
\newblock \emph{arXiv preprint arXiv:1312.6199}, 2013.

\bibitem[Tsipras et~al.(2019)Tsipras, Santurkar, Engstrom, Turner, and
  Madry]{tsipras18}
Tsipras, D., Santurkar, S., Engstrom, L., Turner, A., and Madry, A.
\newblock Robustness may be at odds with accuracy.
\newblock In \emph{International Conference on Learning Representations
  (ICLR)}, volume abs/1805.12152, 2019.

\bibitem[Vershynin(2012)]{rmt}
Vershynin, R.
\newblock \emph{Introduction to the non-asymptotic analysis of random
  matrices}, pp.\  210–268.
\newblock Cambridge University Press, 2012.

\bibitem[Wu et~al.(2021)Wu, Chen, Cai, He, and Gu]{widerwu2021}
Wu, B., Chen, J., Cai, D., He, X., and Gu, Q.
\newblock Do wider neural networks really help adversarial robustness?
\newblock In Beygelzimer, A., Dauphin, Y., Liang, P., and Vaughan, J.~W.
  (eds.), \emph{Advances in Neural Information Processing Systems}, 2021.

\bibitem[Yang et~al.(2020)Yang, Rashtchian, Zhang, Salakhutdinov, and
  Chaudhuri]{closerlook2020}
Yang, Y.-Y., Rashtchian, C., Zhang, H., Salakhutdinov, R.~R., and Chaudhuri, K.
\newblock A closer look at accuracy vs. robustness.
\newblock In \emph{Advances in Neural Information Processing Systems},
  volume~33, pp.\  8588--8601. Curran Associates, Inc., 2020.

\bibitem[Yin et~al.(2019)Yin, Kannan, and Bartlett]{yin2019rademacher}
Yin, D., Kannan, R., and Bartlett, P.
\newblock Rademacher complexity for adversarially robust generalization.
\newblock In \emph{International conference on machine learning}, 2019.

\end{thebibliography}
\bibliographystyle{icml2022}

\clearpage
% \addtocontents{toc}{\protect\setcounter{tocdepth}{2}}
\appendix

\begin{center}
    % \noindent\rule{\textwidth}{4pt} \vspace{-0.2cm}
    
    \Large \textbf{Appendix} % \\ ~\\[-0.5cm]
    % \large \textbf{On the (Non-)Robustness of Two-Layer Neural Networks in Different Learning Regimes} 
    
    % \noindent\rule{\textwidth}{1.2pt}
\end{center}

% \tableofcontents

\section{Justification of our proposed measure of robustness}
\label{sec:justice}
Let us begin by explaining why our proposed measure of robustness based on Dirichlet energy \eqref{eq:sob2} is actually a measure of robustness.

% \paragraph{The gradient correlation matrix.}
Given a continuously-differentiable function $f : \mathbb R^d \to \mathbb R$, consider the $d \times d$ psd matrix $J(f)$ defined by
\begin{eqnarray}
J(f) := \mathbb E_{x \sim N(0,I_d)}[\nabla_x f(x) \otimes \nabla_x f(x)].
\end{eqnarray}
We will see that the spectrum of this matrix plays a key role in quantifying the robustness of $f$, w.r.t the distribution $P$.
% \paragraph{Robustness to random local fluctuactions.}
The following lemma shows that $\mathfrak{S}(f)^2 = \trace(J(f))$, and measures the (non)robustness of $f$ to random local fluctuations in its input.
\begin{restatable}[Derivative of robustness error]{lm}{}
\label{lm:justice}
We have % the identity.
\begin{eqnarray}
\lim_{\delta \to 0^+}\frac{1}{\delta}\mathbb E_{x \sim N(0,I_d)}[\Delta_f(x;\delta)] = \mathfrak{S}(f) = \trace(J(f))^{1/2},
\end{eqnarray}
where $\Delta_f(x;\delta) := \sup_{\|v\|_2\le \delta}|f(x+v)-f(x)|$.
\end{restatable}
This lemma is a direct corollary to Lemma \ref{lm:justification} proved later below.

% \paragraph{Robustness to universal adversarial perturbations.}
The next lemma shows that $\|J(f)\|_{op}$ measures the (non)robustness of $f$ to universal adversarial perturbations, in the sense of \cite{moosavi17}.
\begin{restatable}[Measure of robustness to universal perturbations]{lm}{}
We have the identity
\begin{eqnarray}
\lim_{\delta \to 0^+}\frac{1}{\delta}\Delta_f(\delta) = \|J(f)\|_{op}^{1/2},
\end{eqnarray}
where $\Delta_f(\delta)^2:=\sup_{\|v\| \le \delta}\mathbb E_{x \sim N(0,I_d)}(f(x+v)-f(x))^2$.
\end{restatable}
In particular, the leading eigenvector of $J(f)$ corresponds to (first-order) universal adversarial perturbations of $f$, in the sense of \cite{moosavi17}, which can be efficiently computed using the \emph{Power Method}, for example.

A rough sketch of the proof of the above lemma is as follows. To first-order, we have $f(x+v) - f(x) \approx v^\top \nabla x f(x)$. Thus,
\begin{eqnarray*}
\begin{split}
\Delta_f(\delta)^2 &\approx \sup_{\|v\| \le \delta}\mathbb E_{x \sim N(0,I_d)}[(f(x+v)-f(x))^2]\\
&= \sup_{\|v\| \le \delta}\mathbb E_{x \sim N(0,I_d)}[(v^\top \nabla_x f(x))^2 ]\\
&= \sup_{\|v\| \le \delta} v^\top J(f) v = \delta^2 \|J(f)\|_{op}^2.
\end{split}
\end{eqnarray*}
The first lemma is proved via a similar argument.

\subsection{Why not use  Lipschitz constants to measure robustness ?}
% Note that it always holds that
% \begin{eqnarray}
% \|f\|_{\Lip} \ge \mathfrak{S}(f),
% \end{eqnarray}
Note for any that smooth function $f:\mathbb R \to \mathbb R$, $\mathfrak{S}(f)$ is always a lower-bound for the Lipschitz constant $\|f\|_{\Lip}$ of $f$. Recall that $\|f\|_{\Lip}$ is defined by
\begin{eqnarray}
\|f\|_{\Lip}:= \sup_{x \ne x'}\frac{|f(x)-f(x')|}{\|x-x'\|}.
\end{eqnarray}
One special case where there is equality $\mathfrak{S}(f) = \|f\|_{\Lip}$ is when $f$ is a linear function. However, this is far from true in general: $\|f\|_{\Lip}$ is a worst-case measure, while $\mathfrak{S}(f)$ is an average-case measure for each $q$. 
If $\|f\|_{\Lip}$ is small (i.e., of order $\mathcal O(1)$), then a small perturbation (i.e., of size $\mathcal O(1)$) can only result in mild change in the output of $f$ (i.e., of order $\mathcal O(1)$). However, a large value of $\|f\|_{\Lip}$ is uninformative regarding adversarial examples (for example, one can think of a function which is smooth everywhere except on a set of measure zero). In contrast, a large value for $\mathfrak{S}(f)$ indicates that, on average, it is possible for an adversarial to drastically change the output of $f$ via a small modification of its input.

\paragraph{An illustrative example.}
Consider a quadratic function $f(x):=(1/2)x^\top B x + c$. Note that the ground-truth model $f_\star$ defined in \eqref{eq:truemodel} is of this form.
A direct computation reveals that $\nabla_x f(x) = Bx$ and so $\mathfrak{S}(f)^2 := \mathbb E_{x \sim N(0,I_d)}\|\nabla_x f(x)\|^2 =  \mathbb E_{x \sim N(0,I_d)}\|Bx\|^2 = \|B\|_F^2$. However, the Lipschitz constant of $f$ restricted to the ball of radius $\sqrt{d}$ is\footnote{For fair comparison with our measure of robustness, we restrict the computation of Lipschitz constantt to this ball since $\sqrt{d}$ is the length of a typical random vector from $N(0,I_d)$.},
$$
\|f\|_{\widetilde{\Lip}} = \underset{\|x\| \le \sqrt{d}}{\sup} \|\nabla_x f(x)\| = \underset{\|x\| \le \sqrt{d}}{\sup} \|Bx\| = \sqrt{d}\|B\|_{op},
$$
which can be drastically larger than $\mathfrak{S}(f)=\|B\|_F$. For example, take $B$ to be an ill-conditioned, e.g.,  rank-$1$, matrix.

\subsection{Proofs for Dirichlet energy as a measure of adversarial vulnerability}
\label{sec:justiceproof}
Let $P$ be a probability distribution on $\mathbb R^d$, for example, the gaussian distribution $N(0,I_d)$ assumed in the main article.
Let $\|\cdot\|$ be any norm on $\mathbb R^d$ with dual norm $\|\cdot\|_\star$.
Given a function $f:\mathbb R^d \to \mathbb R$, a tolerance parameter $\delta \ge 0$ (the \emph{attack budget}), and a scalar $q \ge 1$, define $R_{\delta}(f)$ by
\begin{eqnarray}
    R_{q,\delta}(f,g) := \mathbb E_{x \sim P}\left[\Delta_f(x;\delta)^q\right],
\end{eqnarray}

where $\Delta_f(x;\delta) := \sup_{\|x'-x\| \le \delta}|f(x')-f(x)|$ is the maximal variation of $f$ in  a neighborhood of size $\delta$ around $x$.
For $q=2$, we simply write $R_{\delta}(f,g)$ for $R_{\delta,2}(f,g)$.
In particular, $G_{\delta}(f) := \mathbb E[R_{\delta}(f,f_\star)]$ is \emph{adversarial test error} and $G_0(f) := \mathbb E [R_0(f,f_\star)]$ is the ordinary \emph{test error} of $f$, where the expectations are w.r.t all sources of randomness in $f$ and $f_\star$. Of course $G_{\delta}(f)$ is an increasing function of $\delta$.

% \subsection{Approximate measure of adversarial test error}
Define $R_{q,\delta}(f) := R_{q,\delta}(f,f)$ and $R_{\delta}(f) := R_{2,\delta}(f,f)$, which measure the deviation of the outputs of $f$ w.r.t to the outputs of $f$, under adversarial attack. Note $R_{q}(f) \equiv 0$. Also note that in the case where $\|\cdot\|$ is the euclidean $L_2$-norm: if $f$ is a near perfect model (in the classical sense), meaning that its ordinary test error $G_0(f)$ is small, then $R_{\delta}(f)$ is a good approximation for $G_{\delta}(f)$. Finally, (at least for small values of $\alpha$), we can further approximate $R_{\alpha}(f)$ (and therefore $G_{\delta}(f)$, for near perfect $f$) by $\delta^2$ times the Dirichlet energy $\mathfrak{S}(f)^2$. Indeed,
\begin{restatable}{lm}{justification}
\label{lm:justification}
Suppose $f$ is $P$-a.e differentiable and for any $q \in [1,\infty)$, define $\mathfrak{S}_{q}(f)$  by
\begin{eqnarray}
\label{eq:Sq}
    \begin{split}
\mathfrak{S}_{q}(f) &:= \left(\mathbb E_{x \sim P}[\|\nabla_x f(x)\|_\star^q]\right)^{1/q}.
    \end{split}
\end{eqnarray}
(In particular, if $\|\cdot\|$ is the euclidean $L_2$-norm and $q=2$, then $\mathfrak{S}_q(f)^2$ is the Dirichlet energy defined in \eqref{eq:sob2} as our measure of robustness). We have the following

\textbf{(A) General case.} $\mathfrak{S}_{q}(f)$ is the right derivative of the mapping $\delta \mapsto R_{\delta}(f)^{1/q}$ at $\delta = 0$. More precisely, we have the following
\begin{eqnarray}
    \lim_{\delta \to 0^+}\frac{R_{q,\delta}(f)^{1/q}}{\delta} = \mathfrak{S}_{q}(f),
\end{eqnarray}
or equivalently, $R_{q,\delta}(f) = \delta^q\cdot \mathfrak{S}_{P}(f)^q + \text{ Higher order terms in }\delta^q$.

\textbf{(B) Case of Dirichlet energy}
In particular, if $\|\cdot\|$ is the euclidean $L_2$-norm, and we take $q=2$,
\begin{eqnarray}
R_{\delta}(f) = \delta^2\cdot \mathfrak{S}_{P}(f)^2 + \text{ Higher order terms in }\delta^2.
\end{eqnarray}
\end{restatable}

\begin{restatable}{rmk}{}
A heuristic argument was used in \cite{yann2019firstorder} to justify the use of average (dual-)norm of gradient (i.e the average local Lipschitz constant) $\mathbb E_{x \sim N(0,I_d)}[\|\nabla_x f(x)\|_\star]$ (corresponding to $q=1$ in the above) as a proxy for the adversarial generalization.
\end{restatable}

The proof of Lemma \ref{lm:justification} follows directly \emph{Fubini's Theorem} and the following lemma.
\begin{restatable}{lm}{strongslope}
\label{lm:strongslope}
If $f$ is differentiable at $x$, then the function $\delta \mapsto \Delta_f(x;\delta):=\underset{\|x'-x\| \le \delta}{\sup}|f(x')-f(x)|$ is right-differentiable at $0$ with derivative given by $\Delta_f'(x;0) = \|\nabla_x f(x)\|_\star$.
\end{restatable}
% \justification*

\begin{proof}
As $f$ is differentiable, $f(x') = f(x) + \nabla_x f(x)^\top (x'-x) + o(\|x'-x\|)$ around $x$. Therefore for sufficiently small $\delta$, if $B(x;\delta)$ is ball of radius $\delta$ around $x$, then
\begin{eqnarray}
    \begin{split}
\Delta_f(x;\delta) &= \sup_{x' \in B(x;\delta)}\mid \nabla_x f(x)^\top (x'-x)+o(\|x'-x\|)\mid\\
&\le \sup_{x' \in B(x;\delta)}|\nabla_x f(x)^\top(x'-x)| + \sup_{y \in B(x;\delta)} o(\|x'-x\|)\\
&=\|\nabla_x f(x)\|_\star\delta + \sup_{y \in B(x;\delta)} o(\|x'-x\|) \frac{\Delta_f(x;\delta)}{\delta}\\
&\le\|\nabla_x f(x)\|_\star+ \sup_{x' \in B(x;\delta)} \frac{ o(\|x'-x\|)}{\delta}
\end{split}
\end{eqnarray}
% where $\sup_{x' \in B(x;\delta)}|\nabla_x f(x)(x'-x)|=\|\nabla_x f(x)\|_\star\delta$ follows from Cauchy-Schwarz.
Note that $\sup_{x' \in B(x;\delta)} \dfrac{ o(\|x'-x\|)}{\delta}\rightarrow 0$. This proves $\limsup_{\delta \to 0^+}(1/\delta)\Delta_f(x;\delta) \le \|\nabla_x f(x)\|_\star$. Similarly, one computes 
\begin{eqnarray}
    \begin{split}
\Delta_f(x;\delta) &=\sup\mid \nabla_x f(x)^\top (x'-x)+o(\|x'-x\|)\mid \\ &\ge \sup|\nabla_x f(x)^\top (x'-x)| - \sup o(\|x'-x\|) \\ &=\|\nabla_x f(x)\|\delta -\sup o(\|x'-x\|)
    \end{split}
\end{eqnarray}

Hence $\liminf_{\delta \to 0^+}(1/\delta)\Delta_f(x;\delta) \ge \|\nabla_x f(x)\|_\star$, and we conclude that $\delta \mapsto \Delta_f(x;\delta)$ is differentiable at $\delta=0$, with derivative $\Delta_f'(x;0) = \|\nabla_x f(x)\|_\star$ as claimed.
\end{proof}
\begin{proof}[Proof of Lemma \ref{lm:justification}]
By basic properties of limits, one has
\begin{eqnarray}
\begin{split}
\left(\lim_{\delta \to 0^+}\frac{R_{q,\delta}(f)^{1/q}}{\delta^q}\right)^q &= \lim_{\delta \to 0^+}\frac{R_{q,\delta}(f)}{\delta}\\
&= \lim_{\delta \to 0^+}\frac{\mathbb E_{x \sim P}[|\Delta_f(x;\delta)|^q]}{\delta}\\
&= \mathbb E_{x \sim P}\left[\lim_{\delta \to 0^+}\frac{|\Delta_f(x;\delta)|^q}{\delta}\right]\\
&=\mathbb E_{x \sim P}\left[\left(\lim_{\delta \to 0^+}\frac{|\Delta_f(x;\delta)|}{\delta}\right)^q\right]\\
&= \mathbb E_{x \sim P}[\|\nabla_x f(x)\|_\star^q]\\
&:= \mathfrak{S}_q(f)^q,
\end{split}
\end{eqnarray}
where the 3rd line is thanks to \emph{Fubini's Theorem}, and the 5th line is thanks to lemma \ref{lm:strongslope} (and the fact that $\Delta_f(x;0) \equiv 0$). Noting that $R_{q,0}(f) \equiv 0$ then concludes the proof.
\end{proof}

\section{Neural networks at (random) initialization}
\label{sec:init}
We now consider networks at initialization, wherein the hidden weights matrix $W=(w_1,\ldots,w_m)$ is a random $m \times d$ matrix with iid rows from $N(0,\Gamma)$ as in the random features regime \eqref{eq:rf}, but we freeze the output weight vector $z=z_0 \in \mathbb R^m$ at random initialization, with random iid entries from ~$N(0,1/m)$, following standard initialization procedures. Let ~$f_\init$ denote this random network, i.e.,
\begin{eqnarray}
f_\init(x) := (\ainit)^\top \sigma(Wx) = \sum_{j=1}^m \ainit_j \sigma(x^\top w_j).
\label{eq:finit}
\end{eqnarray}
% where the hidden weights matrix $W=(w_1,\ldots,w_m)$ is a random $m \times d$ matrix with iid rows from $N(0,\Gamma)$ and the output $a$ is a random vector from $N(0,(1/m)I_m)$ independent of $W$.
% \subsection{The result}
% We have the following result.
% \begin{mdframed}[backgroundcolor=cyan!10,rightline=false,leftline=false]
\vspace{-.5cm}
\begin{restatable}{thm}{untrained}
\label{thm:untrained}
Under the Conditions \ref{cond:traces} and \ref{cond:growth}, we have the identity in the limit \eqref{eq:proportionate},

% \begin{eqnarray}
$\erob(f_\init)
% := \frac{\mathfrak{S}(f_\init)^2}{\mathfrak{S}(f_\star)^2}
= \dfrac{\|\sigma'\|^2_{L^2(N(0,1))} + \lambda_3^2\|\Gamma\|_F^2/2 + \lambda_2^2\|\Gamma\|_F^2}{4\|B\|_F^2}+ o_{d,\mathbb P}(1),
$
% \end{eqnarray}
where $\lambda_k$ is the $k$th Hermite coefficient of the activation function $\sigma$. In particular, for the quadratic activation function $\sigma(t) = t^2 -1$, we have
% \begin{eqnarray}
$\erob(f_\init)
% := \frac{\mathfrak{S}(f_\init)^2}{\mathfrak{S}(f_\star)^2}
= \dfrac{1+\|\Gamma\|^2_F}{\|B\|^2_F}+ o_{d,\mathbb P}(1).
$
% \end{eqnarray}
\end{restatable}
% \end{mdframed}
% The result is proved in the Appendix \ref{subsec:untrained_proof}.
Analogously, the test error for the NN at initialization is given by the following result.
\begin{restatable}[]{thm}{untrainedgenerr}
Under the Conditions \ref{cond:traces} and \ref{cond:growth}, we have the following identity in the limit \eqref{eq:proportionate},
%\begin{eqnarray}
$\egen(f_\init) = 1 + \dfrac{\|\sigma\|^2_{L^2(N(0,1))} +  \lambda_2^2\|\Gamma\|_F^2/2}{2\|B\|_F^2} + o_{d,\mathbb P}(1).$
% \end{eqnarray}
In particular, for the quadratic activation $\sigma(t) := t^2-1$, we have the following identity
% \begin{eqnarray}
% \begin{split}
$\egen(f_\init) = 1 + \dfrac{1+\|\Gamma\|_F^2}{\|B\|_F^2} + o_{d,\mathbb P}(1).$
% \end{split}
% \end{eqnarray}
\label{thm:untrainedgenerr}
\end{restatable}

% \begin{figure}[!h]
%     \centering
%     \includegraphics[width=.9\linewidth]{init_d=450_n=100000.png}
%     \caption{Normalized test / test error $\egen(f_\init)$ and nonrobustness $\erob(f_\init)$ for neural network trained at initialization. Here, the input-dimension is $d=450$.
%     % and the regularization $\lambda$ is zero.
%     Broken curves correspond to theoretical predictions, while solid curves correspond to actual values observed in experiment}
%     \label{fig:init}
% \end{figure}

% \subsection{Intepretation Theorems \ref{thm:untrained} and \ref{thm:untrainedgenerr}}
% Thm. \ref{thm:untrainedgenerr} reveals that the (normalized) test error $\egen$ is equal to 1 plus the (normalized) nonrobustness $\erob(f_\init)$, i.e
%     \begin{eqnarray}
%     \egen(f_\init) - \erob(f_\init) \to 1.
%     \end{eqnarray}
%     This is in contrast to the generaliation  / robustness tradeoff seen sofar in the other regimes.
%     % \abcomment{only one regime discussed before this section?}.
%     In particular, the relative test error at initialization is no less than~$1$, that is, the model $f_\init$ is at least as inaccurate as the null predictor, while the robustness of the former can be arbitrarily bad (Thm. \ref{thm:untrained}) depending on the covariance matrix of the hidden neurons at initialization. \abcomment{On se perd dans cette phrase.}

% \subsection{The impact of training.}
Combining Thm. \ref{thm:untrainedgenerr} with formula \eqref{eq:sgdgenerr}, we deduce that training a randomly initialized neural network always improves its test error, as one would expect. On the other hand, combining Thm. \ref{thm:sobsgd} and Thm. \ref{thm:untrained}, we deduce that fully training the networks~\eqref{eq:sgd} via SGD:

(1) Degrades robustness if $\|B\|_F^2 \gtrsim \|\Gamma\|_F^2 + 1$. This is because in this case, the parameters of the model align to the signal matrix $B$, which has much larger energy than the parameters at initialization. Indeed, SGD tends to move the covariance structure of the hidden neurons from $\Gamma$ to $B$.
% See Thm. 3 or \cite{Ghorbani19}.

(2) Improves robustness if $\|B\|_F^2 \lesssim \|\Gamma\|_F^2 + 1$.
\section{Miscellaneous}
\subsection{Lazy training of output layer in  RF regime}
We now study the influence of the initialization on the random features regime.
Let $W=(w_1,\ldots,w_m) \in \mathbb R^{m \times d}$ with random rows drawn iid from $N(0,\Gamma)$ as in the RF model \eqref{eq:rf}, and let the output layer be initialized at $z=\ainit \sim N(0,(1/m)I_d)$ and updated via single-pass gradient-flow on the entire data distribution (infinite data). In this so-called random features lazy (RFL) regime, we posit the following approximation neural network \eqref{eq:nn}
\begin{eqnarray}
  f_\rfl(x) := z_{\rfl,\lambda}^\top \sigma(Wx)= f_\init(x) + \delta_\lambda^\top \sigma(W x),
\end{eqnarray}
where $z_{\rfl,\lambda}:=\ainit + \delta_\lambda$ and $\delta_\lambda \in \mathbb R^m$ solves the following ridge-regression problem
\begin{eqnarray}
\begin{split}
\arg\min_{\delta \in \mathbb R^m}\mathbb E_{x \sim N(0,I_d)}[(\delta^\top \sigma(Wx)+f_\init(x)-f_\star(x))^2]+\lambda\|\delta\|^2.
\end{split}
\end{eqnarray}
% \begin{eqnarray}
% \dot \delta(t) = -\varepsilon_{{\rm gen}}(f_{W,\delta(t)}),\,\delta(0) = 0,
% \end{eqnarray}
% and as usual, $\varepsilon_{{\rm gen}}(f)
% := \|f-f_\star\|_{L^2(N(0,I_d))}^2
%:= \mathbb E_{x \sim N(0,I_d)}|f(x)-f_\star(x)|^2$ denotes the test error of a model $f$.
The use of the ridge parameter here can be thought of as a proxy for early-stopping at iteration $t \propto 1/\lambda$ \cite{implicitregTibshirani2020}; $\lambda=0$ corresponds to training the output layer to optimality.
% In terms of second-order statistics\footnote{This is sufficient for analyzing test error and robustness.}, $a(t):=\ainit + \delta(t)$ induces the same neural network as $z_\rfl,\lambda}$ where $\lambda \propto 1/t$. See the recent paper \cite{implicitregTibshirani2020} on implicit regularization of early-stopped least-squares gradient-flow. Thus, $z_{\rfl,\lambda}=\deltz_\lambda + z_0$, where $\delta \in \mathbb R^m$ solves the ridge-regression problem
% $$
% \arg\min_{\delta \in \mathbb R^m}\varepsilon_{{\rm gen}}(f_\rfl) + \lambda\|\delta\|^2
% $$

% \begin{restatable}{rmk}{}
% It may be argued that strictly speaking, the analytic form \eqref{eq:alazy} is not what "lazy training" is doing, since the former is a finite-epoch model while \eqref{eq:alazy} with $\lambda=0$ corresponds to the asymptotic / finite-data limit. We counter-argue that as regards second-order statistics, ridge regularization in general-linear regression is equivalent least-squares gradient-flow, early-stopped at time $t=1/\lambda$, independently of and the distribution of the (here, nonlinear) features. For example, see the recent work by Tibshirani and co-workers. This argument is sufficient, since second order statistics are sufficient for expressing test errors and nonrobustness profiles.
% \end{restatable}

% \subsection{Comparing with random features regime}
% This question is answered by the following result, which will be commented shortly.
\begin{restatable}{thm}{rflthm}
\label{thm:rflthm}
We have the following identities
\begin{align}
    \mathbb E_{\ainit}[\egen(f_{\rfl,\lambda})] &= \egen(f_\rf) + \frac{\trace(P_\lambda^2 U)/m}{2\|B\|_F^2} + o_{d,\mathbb P}(1)\\
    \mathbb E_{\ainit}[\erob(f_{\rfl,\lambda})] &= \erob(f_\rf) + \frac{\trace(P_\lambda^2 C)/m}{4\|B\|_F^2} + o_{d,\mathbb P}(1),
\end{align}
where $U=U(W)$ and $C=C(W)$ are the random matrices defined in \eqref{eq:randomU} and \eqref{eq:randomCbis} respectively.
\end{restatable}
Because $P_\lambda^2$, $U$, and $C$ are psd matrices, the residual terms $\trace(P_\lambda^2 U)/m$ and $\trace(P_\lambda^2 C)/m$ in the above formulae are nonnegative. We deduce that random initialization of the output weights hurts both test error and robustness, as long as the RFL regime is valid.

% \subsection{Detailed analysis of limiting cases}
\textbf{Infinitely regularized case $\lambda \to \infty$.}
Note that $P_\lambda$ converges in spectral norm a.s
%(over the random hidden weights matrix $W$)
to the identity matrix $I_m$ in the limit $\lambda \to \infty$. Thus, in this limit, $z_{\rf,\lambda}$ converges almost-surely to the all-zero $m$-dimensional vector and so, thanks to \eqref{eq:alazy}, the output weights $z_{\rfl,\lambda}$ of $f_{\rfl,\lambda}$ converge to the value at initialization $\ainit$. Therefore, $f_{\rfl,\lambda}$ and all its derivatives converge a.s point-wise its state $f_\init$ at initialization \eqref{eq:finit}.
We deduce that in the $\lambda \to \infty$ limit, the neural network in the lazy regime is equivalent to an untrained model $f_\init$, in terms of test error and robustness.
This does not come as much of a surprise, since $\lambda \to \infty$, corresponds to early-stopping at $t = 0$, i.e., no optimization.

\textbf{Unregularized case $\lambda \to 0^+$.}
By an analogous argument as above, $P_\lambda$ converges a.s. to the all-zero $m \times m$ matrix in the limit $\lambda \to 0^+$, and so thanks to \eqref{eq:alazy}, we have the almost-sure convergence $\|z_{\rfl,\lambda} - z_{\rf,\lambda}\| \to 0$. We deduce that in this limit, the unregularized lazy training regime is exactly equivalent to the unregularized vanilla RF regime.
Thus, the random features lazy (RFL) regime corresponding to the approximation $f_\rfl$ is an interpolation between the random features regime (corresponding to $f_\rf$) and the untrained regime (corresponding to $f_\init$).

Although this is not useful in our infinite data regime, we remark that a non-zero amount of regularization is often crucial for good statistical performance with finite samples. In this, case, $P_\lambda$ is non-zero, and we expect both the test error and robustness to become worse in this lazy RF approximation, compared to vanilla RF.

\subsection{Effect of regularization in RF regime}
Suppose the estimation of the output weights of the RF model is regularized, i.e., for a fixed $\lambda \ge 0$, consider instead the model $f_{\rf,\lambda}(x):= z_{\rf,\lambda}^\top \sigma(Wx)$, where $z_{\rf,\lambda}$ is chosen to solve the following ridge-regularized problem
\begin{eqnarray}
\min_{z \in \mathbb R^m}\|f_{W,z}-f_\star\|_{L^2(N(0,I_d))}^2 +\lambda\|z\|^2.
\end{eqnarray}
A simple computation gives the explicit form
\begin{eqnarray}
z_{\rf,\lambda} = U_\lambda^{-1}v,
\end{eqnarray}
where $U_\lambda:=U+\lambda I_m$, $U=U(W)$ is the random matrix defined in \eqref{eq:randomU}, and $v \in \mathbb R^m$ is random vector defined in \eqref{eq:randomv}.
An inspection of the proof of Theorem \ref{thm:rfratio} (see Appendix~\ref{subsec:rf_proof}) reveals that the situation in the presence of ridge regularization is equivalent to the unregularized case in which we replace $\overline{\lambda}$ by $\overline{\lambda}+\lambda$ in the definition of the matrix $A_0$ which appears in \eqref{eq:psis}. This has the effect of decreasing $\psi_1$ and $\psi_2$, and thanks to \eqref{eq:rfratio}, decreasing the robustness % $\mathfrak{S}(f_{\rf,\lambda})^2$
of the random features model. That is, %\abcomment{Can we just write the formula instead of the previous sentences?}
$\erob(f_{\rf,\lambda})$ is a decreasing function of the amount of regularization of $\lambda$, and in fact, $\underset{\lambda \to \infty}{\lim}\,\erob(f_{\rf,\lambda}) = 0$.

% In the case of standard initialization where $\Gamma = (1/d)I_d$, Theorem \ref{thm:untrained} predicts that $\mathfrak{S}(f_\init)^2 = \Theta(1)$. Consequently and by virtue of Lm. \ref{lm:justice}, on average a perturbation of $L_2$-norm $\mathcal O(1)$ is enough to change the output of $f_\init$ by a constant amount (and therefore fool a binary classifier). Such results have also been established in \cite{mostrelu2020,yann2019firstorder}.
\section{Technical proofs}
\label{sec:technical}
Before proving the main results of the manuscript, we first state and prove some auxiliary results which will be instrumental.

\subsection{Proof of Lemma \ref{lm:formula}: generic formula for (non)robustness of neural network}
Recall the definitions of the approximation error and robustness metrics from Section \ref{subsec:metrics}.
The following lemma was used to express the measure of (non)robustness $\mathfrak{S}(f_{w,z})^2$ of a two-layer neural network $f_{w,z}$ as a quadratic form in the output weights, with coefficient matrix which depends on the distribution of the hidden weights.
\formula*
\begin{proof}
One direcly computes $\nabla_x f_{W,z}(x) =  \sum_{j=1}^m  z_j\sigma'(x^\top w_j)w_j$, and so the Laplacian of $f_{W,z}$ at $x$ is given by
\begin{eqnarray}
\|\nabla_x f_{W,z}(x)\|^2 = \sum_{j,k=1}^m z_jz_k (w_j^\top w_k)\sigma'(x^\top w_j)\sigma'(x^\top w_k).
\label{eq:sobnn}
\end{eqnarray}
Thus, $\mathfrak{S}(f_{W,z})^2$ evaluates to
\begin{eqnarray*}
    \begin{split}
        \mathfrak S(f_{W,z})^2 &:= \mathbb E_{x \sim N(0,I_d)}\|\nabla_x f_{W,z}(x)\|^2 = \sum_{j,k=1}^m z_jz_k (w_j^\top w_k) \mathbb E_x[\sigma'(x^\top w_j)\sigma'(x^\top w_k)] =z^\top C(W) z,
        % \\
       % &= \frac{1}{d}\sum_{j,k=1}^m(w_j^\top w_k)^2  = \frac{1}{d}\|WW^\top\|_F^2.
    \end{split}
\end{eqnarray*}
where the $m \times m$ psd matrix $C(W)$ is as defined in Lemma \ref{lm:formula}. In particular, for the activation function $\sigma(t) := t^2 + s$, one computes
\begin{eqnarray*}
\begin{split}
c_{j,k} &:= (w_j^\top w_k)\mathbb E_{x \sim N(0,I_d)}[\sigma'(x^\top w_j)\sigma'(x^\top w_k)]\\
&= 4(w_j^\top w_k)\mathbb E_{x \sim N(0,I_d)}[(x^\top w_j)(x^\top w_k)] = 4(w_j^\top w_k)^2,
\end{split}
\end{eqnarray*}
where the last step is due to the fact that
\begin{eqnarray*}
\begin{split}
\mathbb E_{x \sim N(0,I_d)}[(x^\top w_j)(x^\top w_k)] &= \mathbb E_x[x^\top w_jw_k^\top x] = \trace(\cov(x)w_jw_k^\top) =
% \trace(\Gamma w_j w_k^\top) =
w_j^\top w_k,
\end{split}
\end{eqnarray*}

by a standard result on the mean of a quadratic form.
\end{proof}

% We will often compare the (non)robustness of a student model, as a ratio of that of the ground truth model.
\begin{restatable}[Robustness of ground-truth model]{cor}{}
It holds that $\mathfrak{S}(f_\star)^2 = 4\|B\|_F^2$.
% In particular, if $\Lambda=I_d$, then $\mathfrak{S}(f_\star)^2 = 4\|B\|_F^2$.
\label{cor:robfstar}
\end{restatable}
\begin{proof}
For the first part follows directly from Lemma \ref{lm:formula} with activation function $\sigma(t) := t^2+b_0/d$ and fixed output weight vector $a=1_m:=(1,\ldots,1)$.
% For the second part, further set $\Lambda=I_d$ to get
% $\mathfrak{S}(f_\star)^2 = 4\sum_{j,k}b_{j,k}^2 = 4\|B\|_F^2$ as claimed.
\end{proof}

\subsection{Approximation of random matrices}
This section establishes some technical results for "linearizing" a number of complicated random matrices which occur in our analysis. We will make heavy use of random matrix theory (RMT) techniques developed in \cite{silverstein95,elkaroui2010,Ledoit2011,Dobriban2018}

We begin by recalling the following definition for future reference.
\keyconstants*

Let $U$ be the random $m \times m$ psd matrix defined in \eqref{eq:randomU} and let $v \in \mathbb R^m$ be the random vector defined in \eqref{eq:randomv}. Recall that $\lambda_k = \lambda_k(\sigma)$ is the $k$ Hermite coefficient of the activation function $\sigma$. Also recall the definition of the scalars $\overline{\lambda}$, $\kappa$, $\tau$, $\overline{\lambda'}$, and $\kappa'$ from \eqref{eq:allcoefs}. The following result was established in \cite{Ghorbani19}.
\begin{restatable}[Lemma 2 of \cite{Ghorbani19}]{prop}{}
If $\lambda_0=0$ and Conditions \ref{cond:traces}, \ref{cond:growth} are in place, then in the limit \eqref{eq:proportionate}, it holds that
\begin{align}
\|U-U_0\|_{op} &= o_{d,\mathbb P}(1),\label{eq:Uapprox}\\
\|v-(\tau/\sqrt{d})1_m\| &= o_{d,\mathbb P}(1),
\end{align}

where the random $m \times m$ psd matrix $U_0$ is defined by
\begin{eqnarray}
U_0:= \overline{\lambda} I_m + \lambda_1^2 WW^\top + (\kappa/d) 1_m1_m^\top + \mu\mu^\top,
\label{eq:randomU0}
\end{eqnarray}

and $\mu =(\mu_1,\ldots,\mu_m) \in \mathbb R^m$ with $\mu_i := \lambda_2\cdot (\|w_i\|^2-1)/2$.
% , $\overline{\lambda} := \mathbb E_G[\sigma(G)^2]-\lambda_1^2$
% , and $\kappa := d\cdot \lambda_2^2 \|\Gamma\|_F^2/2$, and $\tau:=\lambda_2\trace(B\Gamma)d$.
\label{prop:fittedarf}
\end{restatable}

A careful inspection of the proof of the estimate \eqref{eq:Uapprox} reveals that we can remove the condition $\lambda_0 = 0$, at the expense of incurring rank-$1$ perturbations in the matrix $U_0$.
Indeed, let us rewrite $\sigma = \overline{\sigma} + \lambda_0$, and with $\lambda_0(G) = \mathbb E_G[\overline{\sigma}(G)] = 0$ with $G \sim N(0,1)$ independent of the $w_i$'s. Let $T_0$ be the $m \times m$ matrix with entries $(T_0)_{ij}:=\lambda_0(\sigma_i)\lambda_0(\sigma_j)$, where $\sigma_i$ is the function defined by $\sigma_i(z):=\sigma(\|w_i\|z)=\overline{\sigma}_i(z)+\lambda_0$, with $\overline{\sigma}(\|w_i\|z):=\overline{\sigma}(\|w_i\|z)$. Thus, we have the decomposition
\begin{eqnarray}
T_0= \overline{T}_0 + \lambda_0(u 1_m^\top + 1_m u^\top) +  \lambda_0^2 1_m1_m^\top,
\end{eqnarray}
where $u=(\lambda_0(\overline{\sigma}_i))_{i \in [m]}$. Let $\overline{T}_0$ be the $m \times m$ psd matrix with entries $(\overline{T}_0)_{ij}:=\lambda_0(\overline{\sigma}_i)\lambda_0(\overline{\sigma}_j)$. Using the arguments from \cite{Ghorbani19} (since $\lambda_0(\overline{\sigma}) = 0$), one has
\begin{eqnarray}
\|\overline{T}_0-\mu\mu^\top\|_{op} = o_{d,\mathbb P}(1).
\end{eqnarray}

Furthermore, observe that one can write $u1_m^\top = R \mu 1_m^\top$, where $R$ is the $m \times m$ diagonal matrix with $R_{ii} := \lambda_0(\overline{\sigma}_i)/\mu_i$. Now, for large $d$ and any $i \in [m]$, one computes
\begin{eqnarray}
\begin{split}
R_{ii} = \mathbb E_G\left[\frac{\sigma(\|w_i\|G)-\sigma(G)}{\lambda_2\cdot(\|w_i\|^2-1)/2}\right] &= \mathbb E_G\left[\frac{\sigma(\|w_i\|G)-\sigma(G)}{\|w_i\|-1}\right]\cdot \frac{1}{\lambda_2\cdot(\|w_i\|+1)/2} \\
&\to \frac{\mathbb E_G[G\sigma'(G)]}{\lambda_2\cdot 2/2} = \frac{\lambda_2}{\lambda_2}=1.
\end{split}
\end{eqnarray}

We deduce that $\|R-I_m\|_{op} = o_{d,\mathbb P}(1)$, and so
$\|u1_m^\top - \mu1_m^\top\|_{op}=o_{d,\mathbb P}(1)$. This proves the following extension of the above lemma which will be crucial in the sequel.
\begin{restatable}[Linearization of $U$ without the Condition $\lambda_0(\sigma) \ne 0$]{lm}{}
Suppose Conditions \ref{cond:traces} and \ref{cond:growth} are in place.
In the limit \eqref{eq:proportionate}, it holds that
\begin{eqnarray}
\begin{split}
\|U-\widetilde{U}_0\|_{op} &= o_{d,\mathbb P}(1),
\end{split}
\end{eqnarray}
where $\widetilde{U}_0$ is the $m \times m$ random psd matrix given by 
\begin{eqnarray}
\widetilde{U}_0:= \widetilde{\lambda}I_m+\lambda_1^2 WW^\top + (\kappa/d) 1_m1_m^\top + \widetilde{\mu}\widetilde{\mu}^\top, 
\end{eqnarray}
with $\widetilde{\mu} := \mu+\lambda_0 1_m$ and $\widetilde{\lambda}:=\overline{\lambda}-\lambda_0^2 = \mathbb E_{G \sim N(0,1)}[\sigma(G)^2]-\lambda_0^2-\lambda_1^2$.
\label{lm:pertU}
\end{restatable}

% We know that $\|(WW^\top) \odot (WW^\top)-(\|\Gamma\|_F^21_m1_m^\top + I_m)\|_{op} = o_{d,\mathbb P}(1)$.
Let $C=C(W)$ be the random $m \times m$ psd matrix with entries given by
\begin{eqnarray}
c_{ij} := (w_i^\top w_j)\mathbb E_{x \sim N(0,I_d)}[\sigma'(x^\top w_i)\sigma'(x^\top w_j)].
\label{eq:randomC}
\end{eqnarray}

Thanks to Lemma \ref{lm:formula}, we know that $\mathfrak{S}(f_\rf)^2 = z_\rf^\top C z_\rf = v^\top U^{-1}CU^{-1}v$, a random quadratic form in $v$. We start by linearizing the nonlinear random coefficient matrix $C$.
\begin{restatable}[Linearization of $C$]{lm}{}
Suppose Conditions \ref{cond:traces} and \ref{cond:growth} are in place.
Then, in the limit \eqref{eq:proportionate}, we have the following approximation
\begin{eqnarray}
\|C-C_0\|_{op} = o_{d,\mathbb P}(1),
\end{eqnarray}

where $C_0$ is the $m \times m$ random psd matrix given by
\begin{eqnarray}
\begin{split}
C_0 &:= \overline{\lambda'}I_m + (\kappa'/d + \lambda_1^2)WW^\top + (2\kappa/d)1_m1_m^\top,
\end{split}
\end{eqnarray}

with $\kappa' := d\cdot \lambda_3^2 \|\Gamma\|_F^2 / 2 \ge 0$, and $\overline{\lambda'} := \|\sigma'\|_{L^2(N(0,1))}^2 - \lambda_1^2$.
\label{lm:linearizeC}
\end{restatable}
\begin{proof}
Note that $C = (WW^\top) \odot U'$, where $U'$ is the $m \times m$ random psd matrix with entries given by $U'_{ij} := \mathbb E_{x \sim N(0,I_m)}[\sigma'(x^\top w_i)\sigma'(x^\top w_j)]$. 

\emph{-- Step 1: Linearization.}
Invoking the previous lemma with $\sigma'$ in place of $\sigma$, we know that
\begin{eqnarray}
\label{eq:uprimebound}
\|U'-U_0'\|_{op} = o_{d,\mathbb P}(1),
\end{eqnarray}
where $U'_0$ is the $m \times m$ random matrix given by
\begin{eqnarray}
 \begin{split}
U_0' &:= \lambda' I_m + \lambda_1(\sigma')^2 WW^\top + (\kappa(\sigma')/d)1_m1_m^\top + (\mu + \lambda_0(\sigma') 1_m)(\mu + \lambda_0(\sigma') 1_m)^\top\\
&= \lambda' I_m + \lambda_2(\sigma)^2 WW^\top + (\kappa'/d)1_m1_m^\top + (\mu + \lambda_1(\sigma) 1_m)(\mu + \lambda_1(\sigma) 1_m)^\top,
 \end{split}
\end{eqnarray}
and we have used the fact that
$$
\lambda_0((\sigma')^2) - \lambda_0(\sigma')^2 - \lambda_1(\sigma')^2 = \lambda_0((\sigma')^2) - \lambda_1(\sigma)^2 - \lambda_2(\sigma)^2 = \overline{\lambda'} - \lambda_2(\sigma)^2 =: \overline{\lambda'}.
$$
Now, since $\|WW^\top\|_{op} = \mathcal O_{d,\mathbb P}(1)$ by standard RMT, we deduce that from \eqref{eq:uprimebound} that,
\begin{eqnarray}
\begin{split}
\|C-(WW^\top) \odot U'\|_{op} &= \|(WW^\top) \odot (U'-U_0')\|_{op}\\
&\le \|WW^\top\|_{op}\cdot\|U'-U_0'\|_{op}\\
&= o_{d,\mathbb P}(1).
\end{split}
\end{eqnarray}

\emph{-- Step 2: Simplification.} Let $E := \diag((\|w_i\|^2)_{i \in [m]})$ and $F := (WW^\top) \odot (WW^\top)$. Then
\begin{eqnarray}
 \begin{split}
(WW^\top) \odot U'_0  &= \lambda' E + \lambda_2(\sigma)^2 F + (\kappa'/d)WW^\top + 2\lambda_1(\sigma)\diag(\mu)WW^\top  + \lambda_1(\sigma)^2 WW^\top\\
&= \lambda' E + \lambda_1(\sigma)^2 F + (\kappa'/d + \lambda_1(\sigma)^2)WW^\top + 2\lambda_1(\sigma)\diag(\mu)WW^\top.
 \end{split}
 \label{eq:WWTUprime}
\end{eqnarray}
Further, because $\max_{i \in [n]} |\|w_i\|^2-1| = o_{d,\mathbb P}(1)$ by basic concentration, we have
\begin{eqnarray}
\begin{split}
\|E-I_m\|_{op}, \|\diag(\mu)\|_{op} =o_{d,\mathbb P}(1).
\end{split}
\end{eqnarray}

Also, thanks to \cite[Theorem 2.3]{elkaroui2010}, we may linearize $F$ like so
\begin{eqnarray}
\|F-(I_m + \|\Gamma\|_F^21_m1_m^\top)\|_{op} = o_{d,\mathbb P}(1).
\end{eqnarray}

Combining with \eqref{eq:WWTUprime} gives (recalling that $\kappa := d\cdot \lambda_2(\sigma)^2 \|\Gamma\|_F^2/2$)
\begin{eqnarray}
\begin{split}
 (WW^\top) \odot U' &= (\lambda' + \lambda_2(\sigma)^2) I_m + (\kappa'/d+\lambda_1(\sigma)^2) WW^\top + (2\kappa/d) 1_m1_m^\top + \Delta\\
&= \overline{\lambda'}I_m + (\kappa'/d+\lambda_1(\sigma)^2) WW^\top + (2\kappa/d)1_m1_m^\top + \Delta\\
 &= C_0 + \Delta,
 \end{split}
\end{eqnarray}
where $\|\Delta\|_{op} = o_{d,\mathbb P}(1)$.
\end{proof}
Let us rewrite $U_0 = A_1 + \mu\mu^\top$ and $C_0 = D_0 + (2\kappa/d) 1_m1_m^\top$, where \begin{eqnarray}
\begin{split}
A_1 &:= A_0 + (\kappa/d)1_m1_m^\top,\\
A_0 &:= \widetilde{\lambda}I_m + \lambda_1^2 WW^\top,\\
D_0 &:= \overline{\lambda'}I_m + (\kappa'/d+\lambda_1^2)WW^\top.
\end{split}
\label{eq:matrices}
\end{eqnarray}
We will need the following lemmas.
\begin{restatable}{lm}{}
We have the following approximation
\begin{eqnarray}
\begin{split}
\mathfrak{S}(f_\rf)^2 &= u^\top C_0 u + o_{d,\mathbb P}(1)\\
&= \tau^2 \frac{1_m^\top U_0^{-1}C_0 U_0^{-1}1_m}{d} + o_{d,\mathbb P}(1),
\end{split}
\end{eqnarray}
where $u := U_0^{-1} h$, with $h := (\tau/\sqrt{d})1_m = \lambda_2\cdot \trace(B\Gamma)1_m$ and $U_0$ is defined as in Proposition \ref{prop:fittedarf} and $C_0$ is as defined in Lemma \ref{lm:linearizeC}.
\label{lm:underscorezeros}
\end{restatable}
\begin{proof}
Thanks to Proposition \ref{prop:fittedarf}, the fitted output weights vector $z_\rf \in \mathbb R^m$ concentrates around $u:=U_0^{-1} h$.
On the other hand, we know from Lemma \ref{lm:formula} that $\mathfrak{S}(f_\rf)^2 = z_\rf^\top C z_\rf$. The result then follows from Lemma \ref{lm:linearizeC}.
\end{proof}
\begin{restatable}{lm}{}
Under Condition \ref{cond:nonlinear}, the following holds in the limit \eqref{eq:proportionate}
\begin{align}
\frac{1_m^\top U_0^{-1} 1_m}{d} &= \frac{\psi_1}{1+\kappa\psi_1} +  o_{d,\mathbb P}(1)\label{eq:psi},\\
\frac{1_m^\top A_1^{-1} \mu}{\sqrt{d}} &= o_{d,\mathbb P}(1)\label{eq:zero1},\\
\|A_1^{-1}\|_{op},\|D_0\|_{op} &= \mathcal O_{d,\mathbb P}(1)\label{eq:A0invD0norm}.
\end{align}
where $\psi_1>0$ is as defined in \eqref{eq:psis}.% is the unique positive solution of the Silverstein equation for ...
\label{lm:salt}
\end{restatable}
\begin{proof}
Formula \eqref{eq:psi} was established in the proof of \cite[Theorem 1]{Ghorbani19}, whilst \eqref{eq:zero1} was established in the proof of Lemma 5 of the same paper.

As for \eqref{eq:A0invD0norm}, we note that
$$
\|A_1^{-1}\|_{op} = \|(\overline{\lambda} I_m + \lambda_1^2 WW^\top)^{-1}\|_{op} = \mathcal O_{d,\mathbb P}(1/\overline{\lambda},\lambda_1^2) = \mathcal O_{d,\mathbb P}(1),
$$
since $\overline{\lambda} = \Omega_d(1)$ under Condition \ref{cond:nonlinear}. Similarly, one computes $\|D_0\|_{op} = \mathcal O_d(WW^\top) = \mathcal O_{d,\mathbb P}(1)$, by standard RMT arguments \cite{rmt}. 
\end{proof}

We will need one final lemma.
\begin{restatable}{lm}{}
Let $A_1$, $A_0$, and $D_0$ be the random matrices defined in \eqref{eq:matrices}. Then, it holds that
\begin{eqnarray}
\begin{split}
\frac{1_m^\top A_1^{-1}D_0 A_1^{-1}1_m}{d}
% &= \frac{\trace(A_0^{-2} D_0)/d}{(1+\kappa \trace(A_0^{-1})/d)^2} + o_{d,\mathbb P}(1)\\
&= \frac{\psi_2}{(1+\kappa \psi_1)^2} + o_{d,\mathbb P}(1), % (\alpha-\beta)^2\cdot\kappa \trace(A_0^{-2}D_0)/d + o_{d,\mathbb P}(1).
\end{split}
\end{eqnarray}
% where $\psi$ is the unique nonnegative solution of the Silverstein equation xyz.
where $\psi_1$ and $\psi_2$ as defined in \eqref{eq:psis}.
\label{lm:A1D0A1}
\end{restatable}

\begin{proof}
By Sherman-Morrison formula, we have
$$
 A_1^{-1} = A_0^{-1} - \kappa\frac{A_0^{-1}1_m1_m^\top A_0^{-1}/d}{(1 + \kappa 1_m^\top A_0^{-1}1_m/d)},
$$
and so $\dfrac{1_m^\top A_1^{-1}D_0 A_1^{-1}1_m}{d} = a - 2ab - ab^2 = a(1-b)^2=ac^2$, where
\begin{eqnarray}
\begin{split}
a &:= 1_m^\top A_0^{-1} D_0 A_0^{-1} 1_m/d,\\
b &:= \frac{\kappa 1_m^\top A_0^{-1}1_m/d}{1+\kappa 1_m^\top A_0^{-1} 1_m/d},\\
c &:= 1-b = \frac{1}{1+\kappa 1_m^\top A_0^{-1} 1_m/d}.
\end{split}
\end{eqnarray}
Now, one has $1_m^\top A_0^{-1} 1_m/d = \trace(A_0^{-1})/d + o_{d,\mathbb P}(1)$, thanks to Lemmas 5 and 6 of \cite{Ghorbani19}. By an analogous argument, one can show that $1_m^\top A_0^{-1}D_0 1_m/d = \trace(A_0^{-2}D_0)/d + o_{d,\mathbb P}(1)$. Finally, the fact that $\trace(A_0)^{-1}/d$ and $\trace(A_0^{-2}D_0)/d$ converge to deterministic values $\psi_1$ and $\psi_2$ respectively, can be established via standard RMT arguments \cite{silverstein95, Ledoit2011}.
% \begin{eqnarray}
% \begin{split}
% \frac{c^2\tau^2}{\mathfrak{S}(f_\star)^2} &=  \frac{c^2\tau^2}{2\|f_\star\|^2_{L^2(N(0,I_d))}} = \frac{\tau^2/2}{\|f_\star\|^2_{L^2(N(0,I_d))}(1+\kappa \psi)^2} + o_{d,\mathbb P}(1) = \frac{\|f_\star\|^2_{L^2(N(0,I_d))}/2}{\tau^2} \cdot \left(\frac{\tau^2}{\|f_\star\|^2_{L^2(N(0,I_d))}(1+\kappa\psi)}\right)^2 + o_{d,\mathbb P}(1)\\
% &= \frac{\|B\|_F^2\|\Gamma\|_F^2}{\trace(B\Gamma)^2} \cdot \frac{1}{\kappa}\cdot \left(\dfrac{\tau^2}{\|f_\star\|^2_{L^2(N(0,I_d))}(1+\kappa\psi)}\right)^2 + o_{d,\mathbb P}(1).
% \end{split}
% \end{eqnarray}
% But, we know from xyz that 
% $$
% \dfrac{1}{\kappa}\cdot \dfrac{\tau^2}{\|f_\star\|^2_{L^2(N(0,I_d))}(1+\kappa\psi)}=\dfrac{\tau^2}{\kappa \|f_\star\|^2_{L^2(N(0,I_d))}} - (1-\egen(f_\rf)) + o_{d,\mathbb P}(1) = \frac{\trace(B\Gamma)^2}{\|B\|_F^2 \|\Gamma\|_F^2} - (1-\egen(f_\rf)) + o_{d,\mathbb P}(1),
% $$
% and so
% \begin{eqnarray}
% \begin{split}
% \frac{c^2\tau^2}{\mathfrak{S}(f_\star)^2} &= \frac{\trace(B\Gamma)^2}{\|B\|_F^2\|\Gamma\|_F^2} \cdot  \kappa\left(1-\frac{\|B\|_F^2\|\Gamma\|_F^2}{\trace(B\Gamma)^2} \cdot (1-\egen(f_\rf))\right)^2 + o_{d,\mathbb P}(1)\\
% &=\kappa \cdot \alpha^2 (1-\beta/\alpha)^2 + o_{d,\mathbb P}(1) = \kappa\cdot (\alpha-\beta)^2 + o_{d,\mathbb P}(1).
% \end{split}
% \end{eqnarray}
% \begin{itemize}
%     \item[(1)] 
%     \item[(2)] $1_m^\top A_0^{-1} D_0 A_0^{-1}1_m/d = \trace(A_0^{-1} D_0 A_0^{-1})/d + o_{d,\mathbb P}(1)$, by an argument similar to the proof of Lemma 5 of \cite{Ghorbani19}, which employs the \emph{Hanson-Wright inequality}.
% \end{itemize}
\end{proof}

\subsection{Proof of Theorem \ref{thm:rfratio}: Analytic formula for robustness of random features model}
\label{subsec:rf_proof}
We are now ready to prove Theorem \ref{thm:rfratio}, restated here for convenience.
\rfratio*

\begin{proof}%[Proof of Theorem \ref{thm:rfratio}]
From Lemmas \ref{lm:formula} and \ref{lm:linearizeC}, we know that
\begin{eqnarray}
\mathfrak{S}(f_\rf)^2 = z_\rf^\top C z_\rf = u^\top C_0 u + o_{d,\mathbb P}(1) = \tau^2 \frac{1_m^\top U_0^{-1}C_0 U_0^{-1}1_m}{d} + o_{d,\mathbb P}(1),
\end{eqnarray}
where $u := U_0^{-1} h$, with $h := (\tau/\sqrt{d})1_m = \lambda_2\cdot \trace(B\Gamma)1_m$ and $U_0$ defined as in Lemma \ref{lm:pertU} and $C$, $C_0$ are as defined in Lemma \ref{lm:linearizeC}. Let $A_1$, $A_0$, and $D_0$ be the random matrices defined in \eqref{eq:matrices}. Since, $C_0 = D_0 + (2\kappa/d) 1_m1_m^\top$, one computes % In view of applying Lemma \ref{lm:underscorezeros}, one computes

\begin{eqnarray}
\begin{split}
\frac{1_m^\top U_0^{-1}C_0U_0^{-1} 1_m}{d} &= \frac{1_m^\top U_0^{-1}D_0 U_0^{-1} 1_m}{d} + 2\kappa\cdot \frac{1_m^\top U_0^{-1}1_m1_m^\top U_0^{-1}1_m}{d^2}\\
&= \frac{1_m^\top U_0^{-1}D_0 U_0^{-1} 1_m}{d} + 2\kappa \cdot \left(\frac{1_m^\top U_0^{-1} 1_m}{d}\right)^2\\
&= \frac{1_m^\top U_0^{-1}D_0 U_0^{-1} 1_m}{d} + \frac{2\kappa\psi_1^2}{(1+\kappa\psi_1)^2} + o_{d,\mathbb P}(1),% \|B\|_F^4d\cdot\lambda_2^2 \|\Gamma\|_F^2 \cdot \left(\frac{\psi}{\|B\|_F^2(1+\kappa\psi)}+ o_{d,\mathbb P}(1)\right)^2.
\end{split}
\label{eq:allterms}
\end{eqnarray}
where the last step is thanks to Lemma \ref{lm:salt}.
It remains to estimate the first term in the above display.

Using the Sherman-Morrison formula, we have
\begin{eqnarray}
 U_0^{-1} = A_1^{-1} - \frac{A_1^{-1}\mu\mu^\top A_1^{-1}}{1 + \mu^\top A_1^{-1}\mu}.
\end{eqnarray}
We deduce that 
\begin{eqnarray}
 \begin{split}
     \frac{1_m^\top U_0^{-1} D_0 U_0^{-1} 1_m}{d} &= a_{11} - a_{12} - a_{21} + a_{22} + o_{d,\mathbb P}(1),
 \end{split}
 \label{eq:savior}
\end{eqnarray}
where $a_{11}$, $a_{12}$, $a_{21}$, and $a_{22}$ are defined by 
\begin{eqnarray}
 \begin{split}
a_{11} &:= \frac{1_m^\top A_1^{-1}D_0A_1^{-1} 1_m}{d},\\
a_{12} &:=\frac{1_m^\top A_1^{-1} D_0 A_1^{-1}\mu\mu^\top A_1^{-1}1_m}{(1+\mu^\top A_1^{-1}\mu)d},\\
a_{21} &:=\frac{1_m^\top A_1^{-1} D_0 A_1^{-1}\mu\mu^\top A_1^{-1}1_m}{(1+\mu^\top A_1^{-1}\mu)d},\\
a_{22} &:=  \frac{1_m^\top A_1^{-1}\mu\mu^\top A_1^{-1} D_0 A_1^{-1}\mu\mu^\top A_1^{-1} 1_m}{(1+\mu^\top A_1^{-1}\mu)^2 d}.
 \end{split}
\end{eqnarray}
Now, one easily computes
$$
\max(|a_{12}|, |a_{21}|) \le \|D_0\|_{op}\|A_1^{-1}\|_{op}\cdot \frac{1_m^\top A_1^{-1}\mu\mu^\top A_1^{-1}1_m}{(1+\mu^\top A_1^{-1} \mu)d} \lesssim \frac{(1_m^\top A_1^{-1} \mu/\sqrt{d})^2}{(1+\mu^\top A_1^{-1} \mu)} = o_{d,\mathbb P}(1),
$$
where we have used Lemma \ref{lm:salt} in the last two steps. Similarly, we have,
$$
|a_{22}| \le \underbrace{\|D_0\|_{op}\|A_1^{-1}\|_{op}}_{\mathcal O_{d,\mathbb P}(1)}\cdot \underbrace{1_m^\top A_1^{-1} \mu/\sqrt{d}}_{o_{d,\mathbb P}(1)}\cdot \underbrace{\frac{\mu^\top A_1^{-1} \mu}{(1+\mu^\top A_1^{-1}\mu)^2}}_{\mathcal O_{d,\mathbb P}(1)}\cdot \underbrace{\mu^\top A_1^{-1} 1_m/\sqrt{d}}_{o_{d,\mathbb P}(1)} = o_{d,\mathbb P}(1),
$$
again thanks to Lemma \ref{lm:salt}. We conclude from \eqref{eq:savior} that
\begin{eqnarray}
 \frac{1_m^\top U_0^{-1} D_0 U_0^{-1} 1_m}{d} = a_{11} + o_{d,\mathbb P}(1).
 % = \frac{1_m^\top A_1^{-1} D_0 A_1^{-1} 1_m}{d} + o_{d,\mathbb P}(1).
\end{eqnarray}
Finally, we know from Lemma \ref{lm:A1D0A1} that
$$
a_{11} := \dfrac{1_m^\top A_1^{-1} D_0 A_1^{-1} 1_m}{d} = \dfrac{\psi_2}{(1+\kappa\psi_1)^2} + o_{d,\mathbb P}(1).
$$
part (A) of the theorem them follows upon dividing \eqref{eq:savior} by $\mathfrak{S}(f_\star)^2 = 4\|B\|_F^2$.

For part (B), one notes that $\psi_1>0$ and so
\begin{eqnarray*}
\begin{split}
\frac{\tau^2(2\kappa \psi_1^2 + \psi_2)}{\|B\|^2_F(2\kappa \psi_1+2)^2} = \frac{\trace(B\Gamma)^2d(\|\Gamma\|_F^2d\psi_1+\psi_2)}{(\|\Gamma\|_F^2d\psi_1+2)^2\|B\|_F^2} &= \frac{\trace(B\Gamma)^2d^2\|\Gamma\|_F^2d\psi_1}{(\|\Gamma\|_F^2d\psi_1+2)^2\|B\|_F^2} + o_d(1)\\
&= \frac{\trace(B\Gamma)^2}{\|\Gamma\|_F^2\|B\|_F^2} + o_d(1) \to \alpha_\infty^2,
\end{split}
\end{eqnarray*}
which completes the proof.
\end{proof}

% \begin{restatable}{lm}{}
% Let $A_0$ and $D_1$ be as in Lemma xyz. Then, we have the following approximation
% \begin{eqnarray}
% \frac{1_m^\top A_0^{-1}D_1 A_0^{-1}1_m}{d} = \trace(A_0^{-1}D_1A_0^{-1}) + o_{d,\mathbb P}(1).
% \end{eqnarray}
% \end{restatable}

\section{Proofs of main results}
\subsection{Proof of Theorem \ref{thm:untrained}: (Non)robustness of neural network at initialization}
\label{subsec:untrained_proof}
We restate the result here for convenience. Let $f_\init$ be the function computed by the neural network at initialization, as defined in \eqref{eq:finit}.
\untrained*
\begin{proof}
Thanks to Lemma \ref{lm:formula}, we know that
$\mathfrak{S}(f_\init)^2 = z^\top C z$, where $C$ is the random $m \times m$ psd matrix defined in \eqref{eq:randomC}. By standard RMT, $z^\top C z = \trace(C)/m + o_{d,\mathbb P}(1)$. Now, let $C_0$ be the random matrix introduced in Lemma \ref{lm:linearizeC}. Since $\|C-C_0\|_{op} = o_{d,\mathbb P}(1)$ (thanks to the aforementioned lemma), one has $\trace(C)/m = \trace(C_0)/m + o_{d,\mathbb P}(1)$. Let $D_0:=\overline{\lambda'}I_m + (\kappa'/d+\lambda_1^2)WW^\top$ be the matrix defined in \eqref{eq:matrices} so that $C_0 = D_0 + (2\kappa/d)1_m1_m^\top $.
We deduce that in the limit \eqref{eq:proportionate},
\begin{eqnarray}
\begin{split}
 \mathfrak{S}(f_\init)^2 &= \trace(D_0)/m + 2\kappa/d + o_{d,\mathbb P}(1) \\
 &= (\kappa'/d+\lambda_1^2)\trace(WW^\top)/m + \overline{\lambda'} + 2\kappa/d + o_{d,\mathbb P}(1)\\
 % &= (k'/d+\lambda_1^2)\trace(\Gamma) + \overline{\lambda'} +(2\kappa/d)\rho + o_{d,\mathbb P}(1)\\
 &= k'/d+\lambda_1^2 + \overline{\lambda'} + 2\kappa/d +o_{d,\mathbb P}(1)\\
 &= \|\sigma'\|^2_{L^2(N(0,1))} + \kappa'/d + 2\kappa /d + o_{d,\mathbb P}(1)\\
 &= \|\sigma'\|^2_{L^2(N(0,1))} + \lambda_3^2\|\Gamma\|_F^2/2 + \lambda_2^2\|\Gamma\|_F^2 + o_{d,\mathbb P}(1)\\
% &=\begin{cases},
% m\lambda_2^2\|\Gamma\|_F^2) + o_{d,\mathbb P}(1),
% \end{cases}
 \end{split}
\end{eqnarray}

where the third line is because $\trace(WW^\top)/m=(1/m)\sum_{j=1}^m \|w_j\|^2$ which converges in probability to $\trace(\Gamma) = 1$, by the weak law of large numbers. Dividing by both sides of the above display by $\mathfrak{S}(f_\star)^2 = 4\|B\|_F^2$ then gives the result.

In particular, in the case of quadratic activation $\sigma(t) := t^2 - 1$, we have $\lambda_2 = 2$, $\|\sigma'\|_{L^2(N(0,1))} = \lambda_3 = 0$, and so we deduce that
$\mathfrak{S}(f_\init)^2 = 4\|\Gamma\|_F^2$.
\end{proof}

\subsection{Proof of Theorem \ref{thm:untrainedgenerr}: test error of neural network at initialization}
\untrainedgenerr*
\begin{proof}
For random initial output weights $\ainit \sim N(0,(1/m)1_m)$ independent of the (random) hidden weights matrix $W$, one computes
\begin{eqnarray}
\mathbb E_{z}[\varepsilon_{\mathrm{test}}(f_\init)] := \mathbb E_z \mathbb E_{x \sim N(0,I_d)}[(f_\init(x)-f_\star(x))^2] = \mathbb E_z \mathbb E_x [f_\init(x)^2] + \mathbb E_x [f_\star(x)^2],
\end{eqnarray}

where we have used the fact that $\mathbb E z=0$. The second term in the rightmost expression equals $\|f_\star\|_{L^2(N(0,I_d))}^2 = 2\|B\|_F^2$. Let $Q$ be the $m \times m$ diagonal matrix with the output weights $z$ on the diagonal, and let $U$ be the $m \times m$ matrix with entries $U_{ij} := \mathbb E_x[\sigma(x^\top w_j)\sigma(x^\top w_j)]$ introduced in \eqref{eq:randomU}, and let $U_0:= \overline{\lambda} I_m + \lambda_1^2 WW^\top + (\kappa/d) 1_m1_m^\top + \mu\mu^\top$ with $\mu := (\lambda_2(\|w_j\|^2-1))_{j \in [m]} \in \mathbb R^m$, be its approximation given in Proposition \ref{prop:fittedarf}. Then
\begin{eqnarray}
\begin{split}
\mathbb E_x[f_\init(x)^2] &= \mathbb E_x[ \sigma(Wx)^\top Q \sigma(Wx)] = z^\top \mathbb E_x[\sigma(Wx)\sigma(Wx)^\top] z = z^\top U z\\
&= \trace(U)/m + o_{d,\mathbb P}(1),\text {by concentration of random quadratic forms}\\
&= \trace(U_0)/m + o_{d,\mathbb P}(1),\text{ thanks to Proposition \ref{prop:fittedarf}}\\
&= \overline{\lambda} + \lambda_1^2\underbrace{\trace(WW^\top)/m}_{1+o_{d,\mathbb P}(1)} + k/d + \lambda_2 \underbrace{\sum_{i=1}^m(\|w_i\|^2-1)^2/m}_{o_{d,\mathbb P}(1)} + o_{d,\mathbb P}(1)\\
&= \overline{\lambda} + \lambda_1^2 + \kappa/d + o_{d,\mathbb P}(1)\\
&= \|\sigma\|_{L^2(N(0,1))}^2 + \lambda^2_2\|\Gamma\|_F^2/2 + o_{d,\mathbb P}(1).
\end{split}
\end{eqnarray}
The first part of the result then follows upon dividing through by $\|f_\star\|_{L^2(N(0,I_d))}^2 = 2\|B\|_F^2$.

In particular, if $\sigma$ is the quadratic activation, then $\|\sigma\|_{L^2(N(0,I_d))}^2 = \lambda_2 = 2$, and the second part of the result follows.
\end{proof}

\subsection{Proof of Corollary \ref{cor:rfratio}: Random features (RF) regime}
\corrfratio*
\begin{proof}
For quadratic activation, one easily computes
\begin{eqnarray*}
\begin{split}
\lambda_1 &= \lambda_0 = 0,\,\lambda_2 = 2,\,\overline{\lambda} = 2,\, \overline{\lambda'} = 4,\\
\kappa&=\lambda_2^2\|\Gamma\|_F^2d/2 = 2\|\Gamma\|^2_Fd,\,\tau:=2\trace(B\Gamma)/\sqrt{d},\,\kappa' = 0,
\end{split}
\end{eqnarray*}
and because $A_0  = 2I_m$ and $D_0 = 4I_m$ in this case, one has
$$
\psi_1 :=\lim_{\substack{m,d\to \infty\\d/m \to \rho}} \trace(A_0^{-1})/d = \rho / 2\text{ and }
\psi_2 := \lim_{\substack{m,d \to \infty\\d/m \to \gamma}} \trace(A_0^{-1} D_0)/d = \rho.
$$
Plugging these into \eqref{eq:rfratio} yields
\begin{eqnarray*}
\begin{split}
\erob(f_\rf) &= \frac{4\trace(B\Gamma)^2d\cdot 2\cdot 2\|\Gamma\|_F^2 d\cdot (\rho/2)^2}{(2+2\cdot 2\|\Gamma\|_F^2d\cdot \rho/2)^2\|B\|_F^2} + o_{d,\mathbb P}(1)\\
&= \frac{4\trace(B\Gamma)^2\|\Gamma\|_F^2 (\rho d)^2}{(2+2\|\Gamma\|_F^2\rho d)^2\|B\|_F^2} + o_{d,\mathbb P}(1),
\end{split}
\end{eqnarray*}

and all the claims in the corollary follow from Theorem \ref{thm:rfratio}.
\end{proof}

\subsection{Proof of Theorem \ref{thm:rflthm}: Random features lazy (RFL) regime}
\rflthm*
% Let us prove that in the lazy regime, the test error and the nonrobustness are both larger than in the random features regime.
\begin{proof}
By construction, note that the vector $\delta_\lambda$ is equivalent to the output weights of a RF approximation with true labels $\widetilde{f}_\star(x) := f_\star(x) - f_\init(x)$.
If $U$ and $v$ are as defined in \eqref{eq:randomU} and \eqref{eq:randomv} respectively, then we have the closed-form solution (with $U_\lambda := U + \lambda I_m$)
\begin{eqnarray*}
\begin{split}
  \delta_\lambda &= U_\lambda^{-1}(\mathbb E_x[(f_\star(x)-f_{\ainit}(x))\sigma(Wx)])\\
  &= U_\lambda^{-1}(v-\mathbb E_x[(\ainit)^\top \sigma(Wx)\sigma(Wx)^\top])\\
  &= U_\lambda^{-1}(v-U\ainit) =z_{\rf,\lambda} - U_\lambda ^{-1}U\ainit.
  \end{split}
\end{eqnarray*}
Thus, for a fixed regularization parameter $\lambda>0$, the output weights vector in this lazy training regime is given by
\begin{eqnarray}
z_{\rfl,\lambda} = \delta_\lambda + \ainit = z_{\rf,\lambda}+P_\lambda \ainit,
\label{eq:alazy}
\end{eqnarray}where $P_\lambda := I_m-U_\lambda^{-1}U$. We deduce that in the presence of any amount of ridge regularization, the lazy random features (RFL) regime is equivalent to the vanilla random features (RF) regime, with an additive bias of $P_\lambda \ainit \in \mathbb R^m$ on the fitted output weights vector. In particular, note that if $\lambda = 0$, then $z_{\rfl,0} = z_{\rf,0}$, that is in the absence of regularization, the RFL and RF correspond to the same regime (i.e., the initialization has no impact on the final model).

\textit{-- test error.}
From formula \eqref{eq:alazy}, and noting that $\ainit$ is independent of $W$, one computes the test error of $f_{{\rm lazy},\lambda}$ averaged over the initial output weights vector $\ainit$ as 
\begin{eqnarray*}
\begin{split}
\mathbb E_{\ainit}[\varepsilon_{{\rm test}}(f_{{\rm lazy},\lambda})] &:= \mathbb E_{\ainit}[\|f_{{\rm lazy},\lambda}-f_\star\|_{L^2(N(0,I_d))}^2]\\
&=  \|f_{\rf}-f_\star\|_{L^2(N(0,I_d))}^2 + \mathbb E_{\ainit}[\|f_{W,P_\lambda \ainit}\|_{L^2(N(0,I_d))}^2]\\
&= \varepsilon_{{\rm test}}(f_{\rf,\lambda}) + \mathbb E_{a_0}[(\ainit)^\top P_\lambda U P_\lambda \ainit]\\
&= \varepsilon_{{\rm test}}(f_{\rf,\lambda}) + \trace(P_\lambda^2 U)/m,
\end{split}
\end{eqnarray*}
where $U=U(W)$ is the matrix defined in \eqref{eq:randomU}. 

\textit{-- (Non)robustness.}
From formula \eqref{eq:alazy}, one computes
\begin{eqnarray*}
\begin{split}
\mathfrak{S}(f_\rfl,\lambda)^2 = z_{\rfl,\lambda}^\top C z_{\rfl,\lambda} &= z_{\rf,\lambda}^\top C z_{\rf,\lambda} + 2z_{\rf,\lambda} C P_\lambda \ainit + (\ainit)^\top P_\lambda C P_\lambda \ainit\\
&= \mathfrak{S}(f_{\rf,\lambda})^2 + 2z_{\rf,\lambda} C P_\lambda \ainit + (\ainit)^\top P_\lambda C P_\lambda \ainit,
\end{split}
\end{eqnarray*}
where $C=C(W)$ is the matrix defined in \eqref{eq:randomC}.
Taking expectations w.r.t $\ainit$, and noting that $\ainit$ is independent of $P_\lambda$ and $C$ only depend on $W$ and are therefore independent of $\ainit$, we have
\begin{eqnarray}
\mathbb E_{\ainit}[\mathfrak{S}(f_{{\rm lazy},\lambda})^2] = \mathfrak{S}(f_{\rf,\lambda})^2 + \trace(P_\lambda^2 C)/m.
\end{eqnarray}

% \mytodo{Missing step for replacing $U$ by $U_0$ and $C$ by $C_0$ in the traces!}
\end{proof}

\subsection{Proof of Theorem \ref{thm:purent}: Neural tangent (NT) regime}
\purent*
Let $r \le \min(m,d)$ be the rank of $W$. It is clear that $r = \min(m,d)$ w.p $1$. Let
\begin{eqnarray}
W^\top = P_1SV^\top
\label{eq:svdW}
\end{eqnarray}
be the singular-value decomposition of $W^\top$, where $P_1 \in \mathbb R^{d \times r}$ (resp. $V \in \mathbb R^{m \times r}$) is the column-orthogonal matrix of singular-vectors of $W^\top$ (resp. $W$), and $S \in \mathbb R^{r \times r}$ is the diagonal matrix of nonzero singular-values. For any $A \in \mathbb R^{m \times d} $, set $G(A) := SV^\top A \in \mathbb R^{r \times d}$. In their proof of \eqref{eq:generrornt}, \cite{Ghorbani19} showed that it is optimal (in terms of test error) to chose $A_\nt$ such that $G(A_\nt) = P_1^\top B/2$. Multiplying through by the orthogonal projection matrix $P_1$ gives
\begin{eqnarray}
P_1P_1^\top B/2 = P_1G(A_\nt) = P_1SV^\top A_\nt = W^\top A_\nt.
\label{eq:passpass}
\end{eqnarray}

For the proof of Theorem \ref{thm:purent}, we will need the following lemma which was announced in the main paper without proof.

\sobnt*

\begin{proof}
Note that we can rewrite
$$
f_\nt(x) = 2\trace((W^\top A)xx^\top) - c,
$$
which is linear in $xx^\top \in \mathbb R^{d \times d}$.
One then readily computes $\nabla f_\nt(x) = 2(W^\top A + A^\top W)x$, from which we deduce that $\|\nabla f_\nt(x)\|^2 = 4x^\top (W^\top A + A^\top W)^2 x$.
% One readily computes $\nabla f_\nt(x) = 2\sum_{j=1}^m(x^\top a_j)w_j + (x^\top w_j)a_j$, and so
% \begin{eqnarray*}
% \begin{split}
% \frac{1}{4}\|\nabla f_\nt(x)\|^2 = \frac{1}{4}\sum_{j,k=1}^m &(w_j^\top w_k)(x^\top a_j)(x^\top a_k) + (w_j^\top a_k)(x^\top a_j)(x^\top w_k)\\
% &+ (a_j^\top w_k)(x^\top w_j)(x^\top a_k) +  (a_j^\top a_k)(x^\top w_j)(x^\top w_k).
% \end{split}
% \end{eqnarray*}
Averaging over $x \sim N(0,I_d)$ then gives
% \begin{eqnarray*}
% \begin{split}
% \mathfrak{S}(f_\nt)^2 := \mathbb E_{x \sim N(0,I_d)}\|\nabla f_\nt(x)\|^2 = \frac{1}{4}\sum_{j,k} (w_j^\top w_k)(a_j^\top \Lambda a_k) &+ 2(w_j^\top a_k)(a_j^\top \Lambda w_k)
% \\
% &+(a_j^\top a_k)(w_j^\top \Lambda w_k),\\
% % &= \sum_{j,k} \trace(w_j a_k^\top\Lambda a_j w_k^\top) + \trace(a_ja_k^\top \Lambda w_jw_k^\top),
% \end{split}
% \end{eqnarray*}
% In the particular case when $\Lambda = I_d$, the above simplifies to
% \begin{eqnarray*}
% \begin{split}
% \frac{1}{4}\mathfrak{S}(f_\nt)^2 &=  \sum_{j,k=1}^m (w_j^\top w_k)(a_j^\top a_k) + (w_j^\top a_k)(a_j^\top w_k)\\
% &=\trace(WW^\top AA^\top + (WA^\top)^2)\\
% &=\|A^\top W\|_F^2 + \trace((WA^\top)^2)\\
% &=\|WA^\top\|_F^2 + \trace((WA^\top)^2),
% \end{split}
% \end{eqnarray*}
\begin{eqnarray*}
\begin{split}
\frac{\mathfrak{S}(f_\nt)^2}{4} &:= \mathbb E_x\|\nabla f_\nt(x)\|^2 = \mathbb E_x[x^\top (W^\top A + A^\top W )^2 x]\\
&= \trace((W^\top A + A^\top W)^2) = \|W^\top A + A^\top W\|_F^2,
\end{split}
\end{eqnarray*}
which completes the proof.
\end{proof}

We will also need the following auxiliary lemma.
\begin{restatable}{lm}{pureandmixed}
\label{lm:pureandmixed}
Let $P_1$ be as in \eqref{eq:svdW} and let $\beta := \trace(B)^2/(d\|B\|_F^2)$ as usual. In the limit \eqref{eq:proportionate}, we have the identities
\begin{align}
    \mathbb E_W \|P_1P_1^\top B\|_F^2 &= \|B\|_F^2(\underline{\rho} + o_d(1))\label{eq:pure},\\
    \mathbb E_W \|P_1^\top B P_1\|_F^2 &= \|B\|_F^2(\underline{\rho}^2(1-\beta) + \underline{\rho}\beta + o_d(1))\label{eq:mixed},
\end{align}
where $\underline{\rho} := \min(\rho,1)$.
\end{restatable}
\begin{proof}
WLOG, let $B$ be a diagonal matrix, so that $B^2 = \sum_j \lambda_j^2 e_je_j^\top$, where $e_j$ is the $j$th standard unit-vector in $\mathbb R^d$. Then, with $P=P_1P_1^\top$, we have
\begin{eqnarray}
\begin{split}
\|P_1P_1^\top B\|_F^2 &= \trace(PB^2) = \sum_i (PB^2)_{ii} = \sum_{i,j} P_{ij}(B^2)_{ji}\\
&= \sum_{i,j}P_{ij}(B^2)_{ji} = \sum_{i,j}\lambda_i P_{ij}\delta_{ij}^2 = \sum_j \lambda_j P_{jj}.
\end{split}
\end{eqnarray}

Therefore, $\mathbb E_W[\|P_1P_1^\top B\|_F^2]= (1/d)\mathbb E_W[\trace(P)] \cdot \sum_j \lambda_j^2 = \min(m/d,1)\|B\|_F^2 = \|B\|^2(\underline{\rho}+o_d(1))$, where we have used the fact that $\mathbb E_W P_{jj} = (1/d)\mathbb E_W\trace(P)$ for all $j$, due to rotation-invariance. This proves \eqref{eq:pure}.

The proof of \eqref{eq:mixed} is completely analogous to the proof of formula (69) in \cite{Ghorbani19}, with $\rho$ therein replaced with $1-\underline{\rho}$, and is thus omitted.
\end{proof}
\begin{proof}[Proof of Theorem \ref{thm:purent}]
% that it is optimal to take $A$ such that $A^\top W + W^\top A = $ best approximation of $B$ of rank $\le m$.
% is optimal for the test error, in the over-parametrized regime $m \ge d$.

From Lemma \ref{lm:sobnt} and formula \eqref{eq:passpass}), we know that
\begin{eqnarray}
\begin{split}
\mathfrak{S}(f_\nt)^2 &= 4\|W^\top A_\nt + A_\nt^\top W\|_F^2\\
&= 4\|P_1P_1^\top B/2+ BP_1P_1^\top/2\|_F^2\\
&= 2\|P_1P_1^\top B\|_F^2 + 2\|P_1^\top B P_1\|_F^2.
\end{split}
\end{eqnarray}
The result then follows upon taking expectations w.r.t the hidden weights matrix $W$ and applying Lemma \ref{lm:pureandmixed}.
\end{proof}

\subsection{Proof of Theorem \ref{thm:ntl}: Neural tangent lazy (NTL) regime}
\ntL*
\begin{proof}
First observe that $f_\star(x) - f_\ntl(x;A,c) = \widetilde{f}_\star(x) - f_\nt(x;A,c)$, where,
\begin{eqnarray}
\widetilde{f}_\star(x):= f_\star(x) - f_\init(x) =%  x^\top B x - x^\top WW^\top
x^\top \widetilde{B} x + b_0,
\end{eqnarray}
and the $d \times d$ matrix $\widetilde B$ is defined by
\begin{eqnarray}
\widetilde{B} := B - W^\top Q W.
\end{eqnarray}
Thus, fitting the model $f_\ntl(\cdot;A,c)$ to the ground-truth function $f_\star$ with coefficient matrix $B$ is equivalent to fitting $f_\nt(\cdot;A,c)$ to the modified ground-truth $\widetilde{f}_\star$ with coefficient matrix $\widetilde{B}$.

In terms of test error \eqref{eq:generr}, let $A_\ntl$, $c_\ntl$ be optimal in $f_\nt(\cdot;A,c)$, and for simplicity of notation define
\begin{eqnarray}
f_\ntl(x) := f_\ntl(x;A_\ntl,c_\ntl).
\end{eqnarray}

We split the proof into two parts. In the first part, we establish \eqref{eq:ntlsob}. The second part handles \eqref{eq:ntlgen}.

\textit{-- Robustness.}
Proceeding in the same way as in the paragraph leading to \eqref{eq:passpass}, one has
\begin{eqnarray}
\mathfrak{S}(f_\ntl)^2 = 2 \|P_1P_1^\top \widetilde{B}\|_F^2 + 2\mathbb \|P_1^\top\widetilde{B}P_1 \|_F^2,
\label{eq:snlstuff}
\end{eqnarray}
where $P_1 \in \mathbb R^{d \times r}$ is the column-orthogonal matrix in \eqref{eq:svdW} and  $r:=\min(m,d)$ is the rank of $W$ (w.p $1$).
Now, by definition of $\widetilde{B}$, one has $\widetilde{B}^2 = (B-W^\top Q W)(B-W^\top Q W)$, and so
\begin{eqnarray}
\begin{split}
P_1P_1^\top \widetilde{B}^2 = P_1P_1^\top B^2 - P_1P_1^\top B W^\top Q W - P_1P_1^\top W^\top Q W B + P_1P_1^\top W^\top Q W W^\top Q W.
\end{split}
\label{eq:snlstuffbis}
\end{eqnarray}
We now take the expectation w.r.t $(W,\ainit)$, of each term on the RHS. Thanks to Lemma \ref{lm:pureandmixed}, we recognize the expectation w.r.t $W$ of the trace of the first term in \eqref{eq:snlstuffbis} as
\begin{eqnarray}
    \mathbb E_W[\trace(P_1P_1^\top B^2)] = \mathbb E_W[\|P_1P_1^\top B\|_F^2] =  \|B\|_F^2(\underline{\rho} + o_d(1)),
  \end{eqnarray}

Now, since $W$ and $\ainit$ are independent and $\ainit$ has zero mean, the second and third terms in \eqref{eq:snlstuff} have zero expectation w.r.t $(W,\ainit)$ because they are linear in $Q=\diag(\ainit)$.

Finally, one notes that
\begin{eqnarray}
\begin{split}
P_1P_1^\top W^\top Q WW^\top Q W &= P_1 SV^\top D WW^\top QVSP_1^\top = W^\top Q W W^\top Q W,
\end{split}
\end{eqnarray}
and so taking expectation w.r.t $W$ and $D$ (i.e $\ainit$) yields
\begin{eqnarray}
\begin{split}
\mathbb E_{\{W,\ainit\}}[\trace(P_1P_1^\top W^\top Q W W^\top Q W)] &= \mathbb E_{\{W,\ainit\}}[\trace(WW^\top Q W W^\top Q)]\\
&= \mathbb E_{\{W,\ainit\}}[z^\top ((WW^\top) \odot (WW^\top)) z]\\
&= \frac{1}{4}\mathbb E_{\{W,\ainit\}}[\mathfrak{S}(f_\init)^2],
\end{split}
\end{eqnarray}
where the last step is thanks to the second part of Lemma \ref{lm:formula}.
Putting things together, we have at this point established that
\begin{eqnarray}
\mathbb E_{\{W,\ainit\}}[\|P_1P_1^\top \widetilde{B}\|_F^2] = \|B\|_F^2(\underline{\rho} + o_d(1)) + \frac{1}{4}\mathbb E_{\{W,\ainit\}}[\mathfrak{S}(f_\init)^2].
\label{eq:weird1}
\end{eqnarray}

Similarly, noting that $P_1P_1^\top W^\top = W^\top$ by definition of $P_1$, one has
\begin{eqnarray}
\begin{split}
\|P_1\widetilde{B} P_1^\top\|_F^2 &= \trace(P_1P_1^\top \widetilde{B}P_1P_1^\top \widetilde{B}) = \trace((P_1P_1^\top B-W^\top Q W)(P_1P_1^\top B-W^\top Q W))
\\
&= \trace(P_1P_1^\top BP_1P_1^\top)-\trace(P_1P_1^\top B W^\top Q W) - \trace(P_1P_1^\top W^\top W QWB)\\
&\quad\quad\quad\quad\quad\quad\quad\quad\quad\quad\quad\quad\quad\quad\quad\,\,\quad\quad\quad + \trace(W^\top Q WW^\top Q W).
\end{split}
\end{eqnarray}
Taking expectation w.r.t $W$ and $\ainit$ then gives
\begin{eqnarray}
\begin{split}
\mathbb E_{\{W,\ainit\}}\|P_1^\top\widetilde{B}P_1\|_F^2 &= \mathbb E_W [\|P_1^\top B P_1\|_F^2] + \mathbb E_{\{W,\ainit\}}[\trace(WW^\top Q WW^\top Q)]\\
&= \|B\|_F^2(\underline{\rho}^2(1-\beta)+\underline{\rho}\beta + o_d(1)) + \frac{1}{4}\mathbb E_{\{W,\ainit\}}[\mathfrak{S}(f_\init)^2].
\end{split}
\label{eq:weird2}
\end{eqnarray}

Combining \eqref{eq:snlstuff}, \eqref{eq:weird1}, \eqref{eq:weird2}, and \eqref{eq:ntsobsob} then completes the proof of \eqref{eq:ntlsob}.

\textit{-- test error.}
The proof of formula \eqref{eq:ntlgen} build on the proof of Theorem 2 in \cite{Ghorbani19}. 
Let $P_2$ be a $d \times (d-\min(m,d))$ matrix such that the combined columns of $P_1$ and $P_2$ form an orthonormal basis for $\mathbb R^d$. Then, one computes
\begin{eqnarray*}
\begin{split}
\varepsilon_{{\rm test}}(f_\ntl) &:= \|f_\ntl-f_\star\|_{L^2(N(0,I_d))} = \mathbb E_x[|f_\ntl(x)-f_\star(x)|^2]\\
&\overset{(a)}{=} \min_{A \in \mathbb R^{m \times d} } 2\|\widetilde{B} - W^\top A - A^\top W\|_F^2\\
% &= \|\widetilde{B}-P_1P_1^\top \widetilde{B}/2\|_F^2 - \|P_1P_1^\top \widetilde{B}\|_F^2/2 +\|P_1^\top\widetilde{B}P_1\|_F^2/4+\|P_1P_1^\top\widetilde{B}\|_F^2/4\\
% &=\|\widetilde{B}\|_F^2-\|P_1P_1^\top\widetilde{B}\|_F^2+\|P_1P_1^\top\widetilde{B}\|_F^2/4 - \|P_1P_1^\top \widetilde{B}\|_F^2/2 +\|P_1^\top\widetilde{B}P_1\|_F^2/4+\|P_1P_1^\top\widetilde{B}\|_F^2/4\\
&\overset{(b)}{=} 2\|P_2^\top \widetilde{B}P_2\|_F^2 = 2\|P_2^\top(B-W^\top Q W)P_2\|_F^2\\
&\overset{(c)}{=} 2\|P_2^\top B P_2\|_F^2 = \varepsilon_{{\rm test}}(f_\nt).
\end{split}
\end{eqnarray*}
where (a) and (b) are due to arguments analogous to arguments made in the beginning of proof of Theorem 2 in \cite{Ghorbani19} (except that our $\widetilde B$ plays the role of $B$ in \cite{Ghorbani19}) and (c) is because $P_2^\top P_1 = 0 \in \mathbb R^{(d-\min(m,d)) \times d}$ by construction of $P_2$. Dividing through the above display by $\mathfrak{S}(f_\star)^2 = 4\|B\|_F^2$ then gives \eqref{eq:ntlgen}.
\end{proof}

%%%%%%%%%%%%%%%%%%%%%%%%%%%%%%%%%%%%%%%%%%%%%%%%%%%%%%%%%%%%%%%%%%%%%%%%%%%%%%%
%%%%%%%%%%%%%%%%%%%%%%%%%%%%%%%%%%%%%%%%%%%%%%%%%%%%%%%%%%%%%%%%%%%%%%%%%%%%%%%
% APPENDIX
%%%%%%%%%%%%%%%%%%%%%%%%%%%%%%%%%%%%%%%%%%%%%%%%%%%%%%%%%%%%%%%%%%%%%%%%%%%%%%%
%%%%%%%%%%%%%%%%%%%%%%%%%%%%%%%%%%%%%%%%%%%%%%%%%%%%%%%%%%%%%%%%%%%%%%%%%%%%%%%

%%%%%%%%%%%%%%%%%%%%%%%%%%%%%%%%%%%%%%%%%%%%%%%%%%%%%%%%%%%%%%%%%%%%%%%%%%%%%%%
%%%%%%%%%%%%%%%%%%%%%%%%%%%%%%%%%%%%%%%%%%%%%%%%%%%%%%%%%%%%%%%%%%%%%%%%%%%%%%%

\end{document}